\documentclass{article}
\usepackage{iclr2027_conference,times}

\usepackage{amsmath,amsfonts,bm}

\def\1{\bm{1}}

\DeclareMathAlphabet{\mathsfit}{\encodingdefault}{\sfdefault}{m}{sl}
\SetMathAlphabet{\mathsfit}{bold}{\encodingdefault}{\sfdefault}{bx}{n}

\newcommand{\Hc}{\mathbb{H}^{d}_{c}}

\usepackage[T1]{fontenc}
\usepackage[utf8]{inputenc}
\usepackage{amsmath,amssymb,amsthm,mathtools}
\usepackage{booktabs,multirow,makecell,tabularx,array}
\usepackage[table]{xcolor}
\usepackage{graphicx,adjustbox,wrapfig,subcaption}
\usepackage{tikz}
\usetikzlibrary{arrows.meta,calc,positioning,fit,backgrounds,decorations.pathreplacing,patterns,shapes.geometric}
\usepackage{pgfplots}
\pgfplotsset{compat=1.18}
\usepackage{hyperref}
\hypersetup{hidelinks}
\usepackage{url}
\usepackage{enumitem}
\usepackage{threeparttable}
\usepackage[section]{placeins}
\usepackage{float}
\usepackage{siunitx}
\usepackage{pifont}
\usepackage{cleveref}
\usepackage{xspace}
\usepackage{setspace}
\newcommand{\aurora}{\textsc{Aurora}\xspace}
\newcommand{\raoc}{\textsc{RAOC}\xspace}
\usepackage{comment}

\usepackage{titlesec}

\titlespacing*{\section}{0pt}{6pt}{3pt}
\titlespacing*{\subsection}{0pt}{5pt}{2pt}
\titlespacing*{\subsubsection}{0pt}{4pt}{2pt}
\titlespacing*{\paragraph}{0pt}{3pt}{4pt}
\definecolor{auroraBlue}{HTML}{2B6F8E}
\definecolor{auroraTeal}{HTML}{5FA7A0}
\definecolor{auroraMint}{HTML}{DDEFEA}
\definecolor{auroraGold}{HTML}{E6B85C}
\definecolor{auroraRose}{HTML}{D9898D}
\definecolor{auroraGray}{HTML}{EEF1F4}
\definecolor{generalHead}{HTML}{D9EAF2}
\definecolor{fineHead}{HTML}{C6E2E3}
\definecolor{specialHead}{HTML}{ACD3D2}
\definecolor{metricHead}{HTML}{DDEBED}
\definecolor{submetricHead}{HTML}{EAF3F3}
\definecolor{generalCell}{HTML}{EAF2F6}
\definecolor{fineCell}{HTML}{DDEDED}
\definecolor{specialCell}{HTML}{CBE3E2}
\definecolor{vitBand}{HTML}{F3F5F7}
\definecolor{baselineBlue}{HTML}{EAF2F8}
\definecolor{hyperGreen}{HTML}{EAF5EE}
\definecolor{oursRow}{HTML}{D7EFE8}
\definecolor{todoGray}{HTML}{F2F2F2}
\newcommand{\cmark}{\ding{51}}
\newcommand{\xmark}{\ding{55}}

\newcommand{\rot}[1]{\rotatebox{90}{\strut #1}}

\newtheorem{theorem}{Theorem}
\newtheorem{proposition}{Proposition}

\newtheorem{corollary}{Corollary}

\theoremstyle{remark}

\title{Separating Angular Alignment from\\Radial Ordering in Hyperbolic\\Vision--Language Models}

\author{%
\textbf{Mingyang Gao$^{1}$ \quad Yuwei Wu$^{1}$ \quad Shiyu Luo$^{1}$ \quad Wei Wu$^{1}$ \quad Xiaomeng Fan$^{1}$}\\
\textbf{Qihao Lv$^{1}$ \quad Haoyu Zhang$^{1}$ \quad Rui Gao$^{1}$ \quad Bingzhen Liu$^{1}$ \quad Zhi Gao$^{1,*}$ \quad Che Sun$^{2}$}\\
{\normalfont $^{1}$Beijing Institute of Technology}\\
{\normalfont $^{2}$Shenzhen MSU-BIT University}\\
{\normalfont \texttt{gaomingyang@bit.edu.cn} \quad $^{*}$Corresponding author: \texttt{gaozhibit@bit.edu.cn}}\\
{\normalfont \url{https://github.com/Gmy-bit/AURORA}}
}

\iclrfinalcopy

\begin{document}
\maketitle
\lhead{Preprint}


\begin{abstract}
Hyperbolic vision--language models are well suited to modeling the hierarchical structure of multimodal data, because hyperbolic volume grows exponentially with radius, enabling low-distortion representations of tree-like taxonomies from general to specific concepts.
However, existing pretraining objectives for hyperbolic vision--language models entangle semantic alignment with entailment modeling, forcing radial coordinates to simultaneously encode semantic matching and hierarchical structure, leading to conflicting geometric constraints.
In this paper, we introduce \aurora, a hyperbolic vision--language framework that decouples the objectives through angular alignment and a novel \emph{Radial--Angular Order Cone} (\raoc) for entailment modeling. Angular alignment captures semantic similarity while preserving the radial structure for representing hierarchies. Unlike conventional entailment cones, \raoc determines angular allowance from relative parent--child radial gaps. It is incorporated into hierarchical entailment contrastive training to capture partial-order relations. We further introduce a norm regularization term that prevents undesirable radial degeneration for stable training.
We train models on millions of image--text pairs and evaluate them across zero-shot classification, image--text retrieval, hierarchical reasoning, and compositional understanding. Our model consistently outperforms strong Euclidean and hyperbolic vision--language baselines, demonstrating the effectiveness of decoupling semantic alignment from hierarchical entailment modeling. 
\end{abstract}

\section{Introduction}
\label{sec:intro}

Hyperbolic vision--language models, such as MERU~\citep{desai2023meru} and HyCoCLIP~\citep{pal2025hycoclip}, embed images and text in hyperbolic spaces (negatively curved Riemannian manifolds)~\citep{ratcliffe1994foundations} to represent hierarchical relations among visual and linguistic concepts. In the Lorentz model~\citep{cannon1997hyperbolic}, a model of hyperbolic space, distance from the origin serves as a proxy for semantic specificity, with more general concepts represented at smaller radii~\citep{nickel2017poincare}. Under the hyperbolic entailment convention, parent--child relations form a partial order~\citep{vendrov2016order}: abstract parents entail more specific children, whose embeddings lie within their parents' entailment cones~\citep{ganea2018cones}. For example, the text ``\textit{a dog}'' entails the more specific description ``\textit{a dog running on grass}'', and an image depicting ``\textit{a dog running on grass}'' is compatible with both descriptions. Modeling this partial order allows matching visual and textual concepts to occupy different levels of the hierarchy, reflecting their differences in specificity. Hyperbolic volume grows exponentially with radius~\citep{krioukov2010hyperbolic}, accommodating the increasing number of child nodes with depth. Hyperbolic space also enables low-distortion tree embeddings~\citep{sarkar2011low,sala2018tradeoffs}, making it naturally suitable for representing hierarchies.

Learning the hierarchical structures of data requires compatible objectives for semantic alignment and hierarchical entailment. 
However, existing hyperbolic pretraining objectives impose conflicting constraints on radial information: semantic alignment favors similar radii for matched representations, while hierarchical entailment uses radial separation to distinguish parents from children. Moreover, entailment supervision based on positive pairs encourages specified parent--child relations while leaving unpaired relations unpenalized. As a result, semantically unrelated representations may still form spurious hierarchical relations. For example, in Figure ~\ref{fig:cone_comparison}, the text  ``\textit{apple}'' could incorrectly entail an image depicting ``\textit{a rocket successfully lifting off}''.
%

We analyze the geometric compatibility of Lorentz-distance alignment and conventional entailment cones to clarify the optimization conflict, identifying two limitations. Firstly, minimizing Lorentz distance between a matched parent--child pair favors reducing the radial gap that models their difference in specificity.
Conversely, an unmatched child can remain inside a spurious parent's cone even when moved arbitrarily far outward along its radial direction. These cases reveal a mismatch between semantic discrimination and hierarchical entailment. Secondly, conventional cone apertures~\citep{le2019inferring} are inversely coupled to the norm of the parent representation. As this norm approaches the admissible lower limit, the cone becomes less restrictive for partial-order relations. 
\begin{figure}[t]
    \centering
    \includegraphics[width=\linewidth]{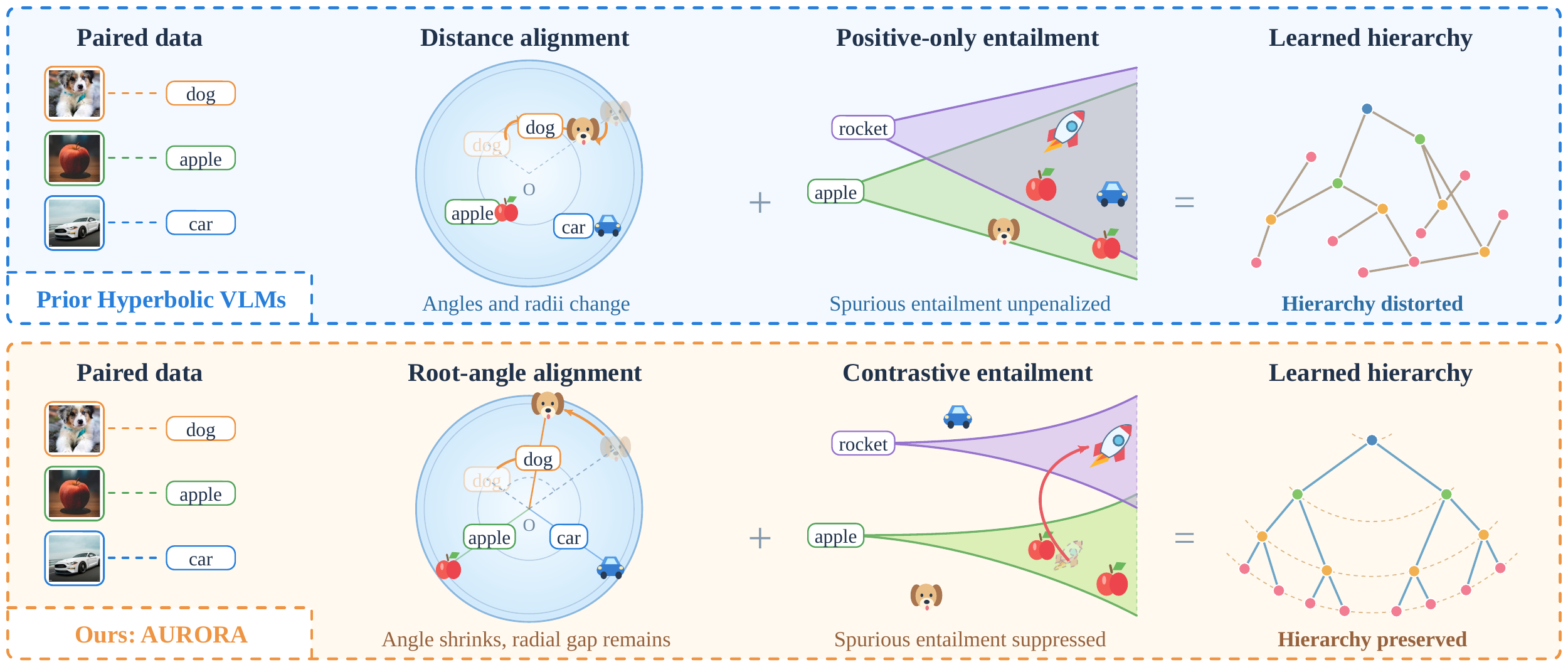} 
    \caption{\textbf{Motivation.}
    \textbf{Top:} Prior hyperbolic VLMs pull paired embeddings together, disrupting radial ordering and entailment. The apple image can even lie closer to the Lorentz origin than ``apple''. 
    \textbf{Bottom:} AURORA preserves radial separation and cleaner entailment: the rocket returns to its own cone, incorrect relations are reduced, and a coherent hierarchy emerges.}
    \label{fig:cone_comparison}
    \vspace{-10pt}
\end{figure}
In this paper, we propose \aurora{}, a pretraining framework that separates \textbf{A}ng\textbf{U}lar alignment from \textbf{R}adial \textbf{O}rdering in hype\textbf{R}bolic vision--l\textbf{A}nguage models. \aurora learns semantic alignment and hierarchical entailment through distinct, complementary objectives. 
For semantic matching, we perform angular alignment using cosine similarity between the spatial directions of Lorentz embeddings, allowing 
their radial positions to model specificity. For hierarchical entailment, we introduce the \emph{Radial--Angular Order Cone}, which controls angular allowance through the parent--child radial gap, replacing the dependence on the parent's norm in conventional cones. 
We prove that the resulting geometric entailment relation defines a partial order 
with nested descendant regions, ensuring that a child's descendants are also descendants of its parent. We further introduce \emph{hierarchical entailment contrastive} training, using 
negative pairs to suppress spurious entailment between unpaired examples. Finally, we present a norm regularization term that prevents undesirable radial degeneration, stabilizing the radial scale during training. Our contributions are summarized as follows:

\begin{itemize}[leftmargin=1.25em,itemsep=.5pt,topsep=1pt,parsep=0pt]
    \item We introduce \aurora, a hyperbolic vision--language pretraining framework that decouples angular semantic alignment from radial modeling of hierarchical entailment.    
    \item We define the \emph{Radial--Angular Order Cone} (\raoc), prove that its geometric entailment relation defines a partial order, and develop hierarchical entailment contrastive (HEC) training on \raoc to better reinforce hierarchical relations.
    \item We train hyperbolic vision--language models with \aurora and demonstrate consistent improvements in zero-shot classification, retrieval, hierarchical reasoning, and compositional understanding over competitive Euclidean and hyperbolic baselines.
\end{itemize}

\section{Preliminaries}
\label{sec:prelim}

\begingroup
\setlength{\abovedisplayskip}{4pt}
\setlength{\belowdisplayskip}{4pt}
\setlength{\abovedisplayshortskip}{2pt}
\setlength{\belowdisplayshortskip}{3pt}

Hyperbolic space is a Riemannian manifold with constant negative
curvature. We adopt the $d$-dimensional Lorentz model with
curvature $-c$, where $c>0$ denotes the curvature
magnitude~\citep{nickel2018lorentz}.
The Lorentz model represents hyperbolic space as the upper
sheet of a two-sheeted hyperboloid in $(d+1)$-dimensional
Minkowski space.

For vectors $\mathbf{x}=[x_0;\tilde{\mathbf{x}}]$ and
$\mathbf{y}=[y_0;\tilde{\mathbf{y}}]$ in $\mathbb R^{d+1}$,
$x_0,y_0\in\mathbb R$ are the time components and
$\tilde{\mathbf{x}},\tilde{\mathbf{y}}\in\mathbb R^d$
are the spatial components.
Their Lorentz inner product is
\begin{equation}
\langle \mathbf{x},\mathbf{y}\rangle_{\mathcal L}
=-x_0y_0+\tilde{\mathbf{x}}^\top\tilde{\mathbf{y}}.
\label{eq:lorentz-inner}
\end{equation}
The Lorentz model is defined as
\begin{equation}
\Hc=\left\{
\mathbf{x}\in\mathbb R^{d+1}:
\langle \mathbf{x},\mathbf{x}\rangle_{\mathcal L}
=-\frac{1}{c},\ x_0>0
\right\},
\label{eq:lorentz-model}
\end{equation}
where $x_0>0$ selects the upper sheet.
Its origin is $O=[1/\sqrt c;\mathbf{0}_d]$, and the
Lorentz distance between two points
$\mathbf{x},\mathbf{y}\in\Hc$ is
\begin{equation}
d_c(\mathbf{x},\mathbf{y})
=\frac{1}{\sqrt c}\operatorname{arcosh}\!\left(
-c\langle \mathbf{x},\mathbf{y}\rangle_{\mathcal L}
\right).
\label{eq:lorentz-distance}
\end{equation}

Following prior hyperbolic vision--language
models~\citep{desai2023meru,pal2025hycoclip},
we treat the projected encoder features
$\mathbf{z}\in\mathbb R^d$ as tangent vectors
$[0;\mathbf{z}]\in T_O\Hc$ at the origin.
The exponential map converts these tangent vectors
into hyperbolic embeddings:
\begin{equation}
\mathbf{x}
=\exp_O^c\!\left([0;\mathbf{z}]\right)
=\frac{1}{\sqrt c}
\left[
\cosh(\sqrt c\|\mathbf{z}\|);
\;
\sinh(\sqrt c\|\mathbf{z}\|)
\frac{\mathbf{z}}{\|\mathbf{z}\|}
\right],
\label{eq:expmap-origin}
\end{equation}
with the zero vector mapped to the origin,
$\exp_O^c(\mathbf{0}_{d+1})=O$.

For a non-origin embedding $\mathbf{x}$, we use the
radial coordinate $\rho_x=\sqrt c\,d_c(O,\mathbf{x})$
and the unit spatial direction
$\mathbf{u}_x=\tilde{\mathbf{x}}/\|\tilde{\mathbf{x}}\|$.
At fixed curvature, larger $\rho_x$ corresponds to
greater distance from $O$.
Their connection to the Lorentz coordinates is provided
in Appendix~\ref{app:rootangle}.

\endgroup

\section{Compatibility of Alignment and Entailment}
\label{sec:analysis}
\begingroup
\setlength{\abovedisplayskip}{4pt}
\setlength{\belowdisplayskip}{4pt}
\setlength{\abovedisplayshortskip}{2pt}
\setlength{\belowdisplayshortskip}{3pt}
\paragraph{Distance alignment couples direction and radius.}
For non-origin embeddings $p,q\in\Hc$, let
$\rho_x=\sqrt c\,d_c(O,x)$ denote radial depth and
$\theta_{pq}=\angle pOq$ their angular separation at the origin.
Writing $\delta_{pq}=\sqrt c\,d_c(p,q)$, the hyperbolic
law of cosines gives
\begin{equation}
\begin{aligned}
\cosh\delta_{pq}
&=\cosh\rho_p\cosh\rho_q
  -\sinh\rho_p\sinh\rho_q\cos\theta_{pq}\\
&=\underbrace{\cosh(\rho_q-\rho_p)}_{\text{radial-gap contribution}}
 +\underbrace{\sinh\rho_p\sinh\rho_q
 (1-\cos\theta_{pq})}_{\text{radius-weighted angular contribution}}.
\end{aligned}
\label{eq:lawcos}
\end{equation}
The first term depends only on the radial gap, and the second
depends on angular separation and is weighted by both radii.
Distance alignment therefore acts on radial depth as well as
semantic direction.
For a parent--child pair with $\rho_q>\rho_p$, holding the
parent and both directions fixed, decreasing the child's
radius toward the parent's decreases their distance.
Even when $\theta_{pq}=0$, the distance remains
$(\rho_q-\rho_p)/\sqrt c$.
Thus, distance minimization continues to penalize the radial
separation that encodes different hierarchical levels.
The derivation is given in
Appendix~\ref{app:alignment-analysis}.

\paragraph{Conventional cones depend on more than pairwise distance.}
Conventional hyperbolic entailment cones admit descendants
within an angular region extending outward from the
parent~\citep{ganea2018cones,desai2023meru}.
Their membership is not determined by distance:
two distinct embeddings at the same radial depth can be
arbitrarily close yet remain outside each other's cones,
whereas a child on the parent's outward radial ray remains
inside its cone at arbitrarily large distances.
Consequently, bringing a matched pair closer does not ensure
entailment, and pushing an unrelated pair farther apart does
not ensure its exclusion from the descendant region.
With positive-only entailment supervision, such erroneous
relations receive no direct entailment penalty.
Cone selectivity also depends on the parent's absolute depth:
for a fixed cone-width setting, the half-aperture satisfies
$\sin\omega(p)\propto1/\sinh\rho_p$.
As the parent moves inward toward its admissible lower radius,
$\omega(p)$ approaches $\pi/2$, admitting the entire outward
hemisphere of geodesic directions at the boundary.
Parent contraction can therefore weaken angular discrimination.
The distance counterexamples and aperture analysis are provided
in Appendix~\ref{app:alignment-analysis}.

Together, these observations motivate semantic alignment based
on direction alone, an entailment relation defined through
parent--child radial gaps and angular differences, and direct
entailment supervision for both positive and negative pairs.

\begin{figure}[t]
    \centering
    \includegraphics[width=1\linewidth]{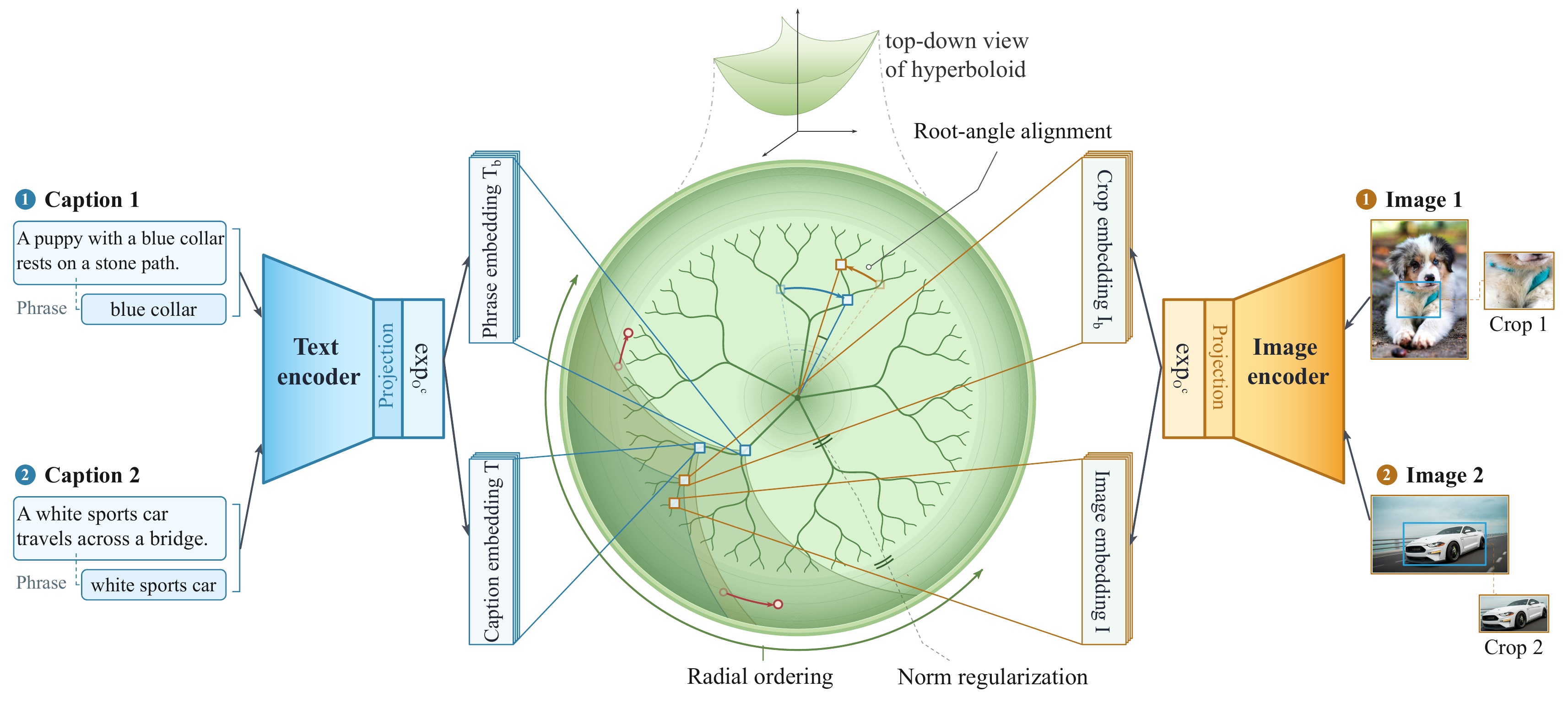}
    \caption{\textbf{Overview of \aurora.} Root-angle alignment brings related visual and textual semantics together, while HEC training via RAOC organizes concepts radially from general to specific.}
    \label{fig:aurora_overview}
    \vspace{-8pt}
\end{figure}
\endgroup
\section{AURORA}
\label{sec:method}
\begingroup
\setlength{\abovedisplayskip}{4pt}
\setlength{\belowdisplayskip}{4pt}
\setlength{\abovedisplayshortskip}{2pt}
\setlength{\belowdisplayshortskip}{3pt}

\subsection{Formulation}
\label{sec:architecture}

Following the training setup of HyCoCLIP~\citep{pal2025hycoclip}, we use image--caption pairs $(I,T)$ with object-level annotations. A bounding box in $I$ specifies an object crop $I^b$, and the corresponding phrase $T^b$ describes its content. Depicted in Figure \ref{fig:aurora_overview}, AURORA uses a visual encoder $f_v$ and a text encoder $f_t$, followed by projections $P_v$ and $P_t$, respectively. Images and crops share the visual branch, while captions and phrases share the text branch. For an input $X\in\{I,T,I^b,T^b\}$, let $f$ and $P$ denote the corresponding encoder and projector. We obtain the projected feature $z_X=P(f(X))$ and its hyperbolic representation $x_X=\exp_O^c(0,z_X)$.

Hierarchy supervision uses the relations types $\mathcal R=\{T\!\to\!I,\;T^b\!\to\!I^b,\;I^b\!\to\!I,\;T^b\!\to\!T\}$, with arrows pointing from general to specific. Text is treated as more general than its visual counterpart, which contains additional instance details. Likewise, crops and phrases are treated as more general than full images and captions, which add scene context. We align matching visual and textual representations through root-angle alignment (Sec.~\ref{sec:root-alignment}) and learn the hierarchical relations in $\mathcal R$ through RAOC and HEC training (Secs.~\ref{sec:raoc}--\ref{sec:contrastive-entailment}). Norm regularization controls their radial scale.

\subsection{Root-angle alignment}
\label{sec:root-alignment}

For two non-origin hyperbolic representations $x,y\in\Hc\setminus\{O\}$ from either modality, let $\theta_{xy}=\angle xOy$ be their root angle. Write $z_x$ and $z_y$ for the projected features that generate $x$ and $y$ through the origin exponential map. This map preserves their directions, giving $u_x=z_x/\|z_x\|$ and $u_y=z_y/\|z_y\|$. We measure semantic similarity by $s_{\angle}(x,y)=\cos\theta_{xy}=u_x^\top u_y$. This score depends only on direction, leaving radial differences available for hierarchical ordering.

We optimize $s_{\angle}$ with a pairwise sigmoid loss.
For embedding batches $A=\{a_i\}_{i=1}^{N}$ and
$B=\{b_j\}_{j=1}^{N}$, set $y_{ij}=+1$ for matching pairs
($i=j$) and $y_{ij}=-1$ otherwise.
With $\sigma(t)=(1+e^{-t})^{-1}$,
\begin{equation}
\mathcal L_{\sigma}(A,B)
=-\frac1N\sum_{i=1}^{N}\sum_{j=1}^{N}
\log\sigma\!\left(y_{ij}\left[
a_m s_{\angle}(a_i,b_j)+b_m
\right]\right),
\label{eq:alignment-sigmoid}
\end{equation}
where $a_m>0$ and $b_m$ are a learned scale and bias. The complete alignment objective is
\begin{equation}
\mathcal L_{\mathrm{align}}=\frac14\left[\mathcal L_\sigma(I,T)+\mathcal L_\sigma(I^b,T^b)+\mathcal L_\sigma(I^b,T)+\mathcal L_\sigma(T^b,I)\right].
\label{eq:alignment-loss}
\end{equation}
The first two terms align corresponding visual and textual content, while crop--caption and phrase--image matching connect object-level semantics to the full scene. All similarity-based downstream evaluations use $s_{\angle}$.

\subsection{Radial--Angular Order Cone}
\label{sec:raoc}

Root-angle alignment measures semantic agreement, while
hierarchical entailment additionally requires an order
from general to specific.
For a candidate parent $p$ and descendant $q$ in
$\Hc\setminus\{O\}$, let
$\rho_x=\sqrt c\,d_c(O,x)$ denote scaled radial depth.
We define the signed radial gap
$g(p,q)=\rho_q-\rho_p$ and root angle
$\theta_{pq}=\angle pOq$.
The \emph{Radial--Angular Order Cone} (RAOC) defines
the admissible descendants of $p$:
\begin{equation}
\mathcal C_\beta(p)=
\left\{
q\in\Hc\setminus\{O\}:
g(p,q)\ge\frac{\theta_{pq}}{\beta}
\right\},
\qquad \beta>0.
\label{eq:raoc-region}
\end{equation}

\begin{wrapfigure}{r}{0.32\columnwidth}
  \vspace{-12pt}
  \centering
  \includegraphics[width=\linewidth]{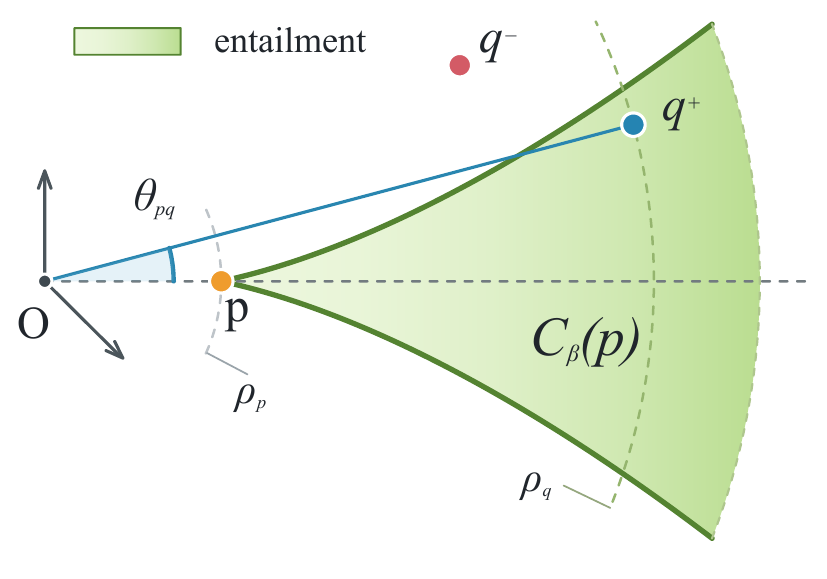}
  \caption{\textbf{RAOC.} Greater angular separation
  requires a larger radial gap.}
  \label{fig:raoc}
\end{wrapfigure}

The required radial gap grows with angular separation:
a larger directional difference requires the child to lie
farther outward from its parent (Figure ~\ref{fig:raoc}).
The fixed hyperparameter $\beta$ controls the angular
allowance at each gap, with larger values giving a looser
constraint.
RAOC therefore ties angular selectivity to the relative
parent--child radial gap rather than the absolute parent
radius.
It defines a partial order with nested descendant regions,
so a child's descendants are also descendants of its parent.
The formal statement and proof are given in
Appendix~\ref{app:raoc-proof}.

For optimization, we formulate the RAOC constraint in terms
of two bounded margins.
The radial margin measures radial precedence, while the
angular margin measures the admissible angular allowance
minus the observed angle.
We define
\begin{equation}
m_r(p,q)=
\operatorname{clip}\!\left(
\frac{g(p,q)}{\sigma_r},-1,1
\right),
\qquad
m_a(p,q)=
\frac{
\operatorname{clip}(\beta g(p,q),0,\pi)-\theta_{pq}
}{\pi}.
\label{eq:raoc-margins}
\end{equation}
Here $\sigma_r>0$ sets the radial margin scale, and
$\operatorname{clip}(t,a,b)=\min\{b,\max\{a,t\}\}$.
Both margins lie in $[-1,1]$.
For negative gaps, the angular allowance is clipped to zero,
while $m_r<0$ identifies the incorrect radial order even
when $\theta_{pq}=0$.

The entailment score is the smaller of the two margins:
\begin{equation}
D_{\mathrm{RAOC}}(p,q)=
\min\{m_r(p,q),m_a(p,q)\}.
\label{eq:raoc-score}
\end{equation}
Thus, $D_{\mathrm{RAOC}}(p,q)\ge0$ requires both radial
precedence and angular compatibility and is equivalent to
$q\in\mathcal C_\beta(p)$.
Further details are provided in
Appendix~\ref{app:raoc-score-note}.

\subsection{Hierarchy learning via contrastive entailment}
\label{sec:contrastive-entailment}

Hierarchical Entailment Contrastive (HEC) training reinforces annotated parent--child relations and suppresses spurious relations between nonmatching examples. For each relation type $U\!\to\!V\in\mathcal R$, $x_{U_i}$ and $x_{V_i}$ form a positive pair, while $x_{U_i}$ and $x_{V_j}$ with $i\ne j$ form a negative pair. Using the labels $y_{ij}$ defined in Sec.~\ref{sec:root-alignment}, we optimize
\begin{equation}
\mathcal L_{\mathrm{HEC}}=-\frac{1}{|\mathcal R|N}\sum_{U\to V\in\mathcal R}\sum_{i,j=1}^{N}\log\sigma\!\left(y_{ij}\left[a_eD_{\mathrm{RAOC}}(x_{U_i},x_{V_j})+b_e\right]\right),
\label{eq:hce-loss}
\end{equation}
where $a_e>0$ and $b_e$ are a learned scale and bias. This objective increases entailment compatibility for matching parent--descendant pairs and decreases it for nonmatching pairs.

Angular alignment leaves radial position unconstrained, while RAOC organizes representations through relative radial gaps. We stabilize their radial scale with the norm penalty $\ell_{\mathrm{norm}}(x)=\widetilde r_x^{\,2}-\log\widetilde r_x$, where $\widetilde r_x=\max\{\|x_s\|,\epsilon\}$ and $\epsilon>0$ keeps the logarithm finite. The quadratic and negative logarithmic terms penalize large and small spatial norms, respectively. We define $\mathcal L_{\mathrm{norm}}$ as the average of this penalty over all image, caption, crop, and phrase embeddings in the minibatch. The complete objective is
\begin{equation}
\mathcal L=\mathcal L_{\mathrm{align}}+\lambda_e\mathcal L_{\mathrm{HEC}}+\lambda_n\mathcal L_{\mathrm{norm}},
\label{eq:total}
\end{equation}
where $\lambda_e$ and $\lambda_n$ weight entailment learning and norm regularization. Implementation details are provided in Appendix~\ref{app:implementation}. Further details on hyperparameter selection and sensitivity analysis are provided in Appendix~\ref{app:hyperparameter-sensitivity}.

\endgroup
\section{Experiments}
\label{sec:experiments}

\subsection{Training Details}
\label{sec:training-details}

We pretrain AURORA on GRIT~\citep{peng2023kosmos2}, comprising 10.3M valid image--caption pairs and 17.8M grounded box--phrase annotations. We mainly report results using ViT-S/16, with complete comparisons for both ViT-S/16 and ViT-B/16 provided in Appendix~\ref{app:fulltables}. Training schedules, hyperparameters, and implementation details are given in Appendix~\ref{app:implementation}. We compare AURORA with CLIP~\citep{radford2021clip}, SigLIP~\citep{zhai2023siglip}, MERU~\citep{desai2023meru}, ATMG~\citep{ramasinghe2024accept}, HyCoCLIP~\citep{pal2025hycoclip}, PHyCLIP~\citep{yoshikawa2026phyclip}, and ARGENT~\citep{huynh2026argent}. For CLIP, SigLIP, and MERU, we include variants pretrained with and without box--phrase pairs to assess the contribution of additional grounded supervision. All comparisons are conducted under the same task-specific evaluation protocols. AURORA uses root-angle similarity $s_{\angle}$ for all similarity-based evaluations.

\subsection{Experimental Results}
\label{sec:experimental-results}

\paragraph{Zero-shot image classification.}
We evaluate zero-shot transfer on 16 datasets covering general, fine-grained, and specialized recognition tasks. As shown in Table~\ref{tab:classification-main}, AURORA achieves the highest accuracy on 13 datasets. Its advantage extends to comparisons with box-supervised CLIP, SigLIP, and MERU, suggesting that the improvement cannot be attributed solely to the inclusion of additional grounded training pairs.
\vspace{-8pt}
\begin{table}[H]
\centering
\caption{\textbf{Zero-shot image classification on ViT-S/16.} Results are reported as top-1 accuracy (\%). \textbf{Bold} indicates the best results. \underline{Underlining} denotes the second best.}

\label{tab:classification-main}
\vspace{-2pt}
\scriptsize
\setlength{\tabcolsep}{2.05pt}
\renewcommand{\arraystretch}{1.04}
\resizebox{\textwidth}{!}{%
\begin{tabular}{@{}l c *{16}{r}@{}}
\toprule
\multicolumn{2}{c}{\cellcolor{vitBand}\textbf{ViT-S/16}} &
\multicolumn{6}{c}{\cellcolor{generalHead}\textbf{General datasets}} &
\multicolumn{6}{c}{\cellcolor{fineHead}\textbf{Fine-grained datasets}} &
\multicolumn{4}{c}{\cellcolor{specialHead}\textbf{Specialized datasets}}\\
\cmidrule(lr){3-8}\cmidrule(lr){9-14}\cmidrule(lr){15-18}
\textbf{Model} & \rot{w/ boxes} &
\cellcolor{generalCell}\rot{ImageNet} &
\cellcolor{generalCell}\rot{CIFAR-10} &
\cellcolor{generalCell}\rot{CIFAR-100} &
\cellcolor{generalCell}\rot{SUN397} &
\cellcolor{generalCell}\rot{Caltech-101} &
\cellcolor{generalCell}\rot{STL-10} &
\cellcolor{fineCell}\rot{Food-101} &
\cellcolor{fineCell}\rot{CUB} &
\cellcolor{fineCell}\rot{Cars} &
\cellcolor{fineCell}\rot{Aircraft} &
\cellcolor{fineCell}\rot{Pets} &
\cellcolor{fineCell}\rot{Flowers} &
\cellcolor{specialCell}\rot{DTD} &
\cellcolor{specialCell}\rot{EuroSAT} &
\cellcolor{specialCell}\rot{RESISC45} &
\cellcolor{specialCell}\rot{Country211}\\
\midrule
CLIP & \xmark & 32.66&74.19&42.22&42.91&69.80&90.09&42.72&12.55&3.93&2.06&43.22&18.03&17.57&25.47&33.64&3.88\\
CLIP & \cmark & 32.49&73.69&42.15&42.66&70.07&89.83&42.50&12.48&4.01&2.10&43.28&17.79&17.61&26.04&33.37&3.84\\
SigLIP & \xmark & 34.30&66.57&38.38&44.73&71.42&91.75&38.98&13.33&4.89&2.13&46.65&\underline{23.17}&18.40&\textbf{32.63}&35.42&4.15\\
SigLIP & \cmark & 34.34&67.07&38.50&44.82&71.55&91.82&39.27&13.33&4.92&2.24&46.56&23.03&18.36&\underline{31.42}&35.37&\underline{4.23}\\
MERU & \xmark & 31.74&76.08&41.24&41.28&67.76&87.24&41.09&11.97&3.80&\underline{2.88}&42.53&16.75&17.46&27.75&33.88&3.82\\
MERU & \cmark & 32.01&76.55&41.35&41.21&67.64&87.30&41.24&12.08&3.78&\textbf{2.95}&42.54&16.57&17.52&27.53&33.82&3.83\\
ATMG & \xmark & 33.74&77.21&43.28&43.15&68.08&90.84&44.01&13.27&4.38&1.75&42.53&15.70&17.89&28.69&35.75&4.21\\
HyCoCLIP & \cmark & 36.38&83.28&47.91&47.88&71.53&93.16&46.73&15.33&5.98&2.45&45.59&20.89&19.47&26.00&38.73&4.09\\
PHyCLIP & \cmark & 37.06&83.44&48.41&48.88&71.82&93.06&47.26&16.29&6.58&2.57&46.44&21.46&19.13&25.81&39.79&3.95\\
ARGENT & \cmark & \underline{37.55}&\underline{83.55}&\underline{48.62}&\underline{49.29}&\underline{72.15}&\textbf{94.01}&\underline{47.70}&\underline{16.90}&\underline{6.81}&2.48&\underline{47.10}&21.65&\underline{19.67}&26.32&\underline{40.46}&4.19\\
\rowcolor{oursRow}\textbf{AURORA} & \cmark & \textbf{41.54}&\textbf{84.25}&\textbf{51.20}&\textbf{50.27}&\textbf{75.81}&\underline{93.19}&\textbf{50.13}&\textbf{17.71}&\textbf{8.91}&1.80&\textbf{51.41}&\textbf{28.99}&\textbf{22.87}&30.18&\textbf{41.09}&\textbf{4.70}\\
\bottomrule
\end{tabular}}
\end{table}
\vspace{-8pt}
\paragraph{Image--text retrieval and hierarchical classification.}
We evaluate bidirectional image--text retrieval on COCO~\citep{lin2014coco} and Flickr30K~\citep{young2014flickr}, together with hierarchical classification on ImageNet~\citep{deng2009imagenet}. The latter uses the WordNet taxonomy to assess how closely predictions agree with ground-truth labels in the hierarchy. Table~\ref{tab:retrieval-hierarchy-main} shows that AURORA achieves the best results in both retrieval directions and across all reported hierarchical metrics. Thus, improved cross-modal matching is accompanied by stronger agreement with the label hierarchy. We follow the evaluation protocols of HyCoCLIP and PHyCLIP, with metric definitions detailed in Appendix~\ref{app:protocols}.

\paragraph{Compositional understanding.}
We evaluate compositional understanding on VL-CheckList-Object~\citep{zhao2022vlchecklist} and SugarCrepe~\citep{hsieh2023sugarcrepe}, where models must distinguish a correct caption from a hard negative containing altered semantic content. VL-CheckList-Object tests sensitivity to object replacement across different object locations and sizes. SugarCrepe extends the evaluation to seven types of perturbations involving objects, attributes, and relations.

As shown in Table~\ref{tab:composition-main}, AURORA leads on nine of the fourteen reported measures, including the overall SugarCrepe score. The most pronounced gains occur on VL-CheckList-Object, suggesting more reliable identification of objects across variations in location and scale. The best overall performance on SugarCrepe further indicates improved discrimination of compositional mismatches beyond object replacement. 
\vspace{-8pt}
\begin{table}[H]
\centering
\caption{\textbf{Zero-shot retrieval and hierarchical classification on ViT-S/16.} Retrieval performance is reported as Recall@5/10 (\%). Hierarchical metrics are shown to two decimal places, with rankings determined before rounding. \textbf{Bold} indicates the best results. \underline{Underlining} denotes the second best.}

\label{tab:retrieval-hierarchy-main}
\vspace{-4pt}
\scriptsize
\setlength{\tabcolsep}{2.55pt}
\renewcommand{\arraystretch}{1.05}

\resizebox{\linewidth}{!}{%
\begin{tabular}{@{}l c *{13}{c}@{}}
\toprule

\multicolumn{2}{c}{\cellcolor{vitBand}\textbf{ViT-S/16}} &
\multicolumn{4}{c}{\cellcolor{generalHead}\textbf{Text $\rightarrow$ Image}} &
\multicolumn{4}{c}{\cellcolor{fineHead}\textbf{Image $\rightarrow$ Text}} &
\multicolumn{5}{c}{\cellcolor{specialHead}\textbf{Hierarchical Classification}}\\

\cmidrule(lr){3-6}
\cmidrule(lr){7-10}
\cmidrule(lr){11-15}

\textbf{Model} & &
\multicolumn{2}{c}{\cellcolor{generalCell}\textbf{COCO}} &
\multicolumn{2}{c}{\cellcolor{generalCell}\textbf{Flickr}} &
\multicolumn{2}{c}{\cellcolor{fineCell}\textbf{COCO}} &
\multicolumn{2}{c}{\cellcolor{fineCell}\textbf{Flickr}} &
\multicolumn{5}{c}{\cellcolor{specialCell}\textbf{WordNet}}\\

& \raisebox{0pt}[0pt][0pt]{\shortstack[c]{w/\\boxes}} &
R@5 & R@10 & R@5 & R@10 &
R@5 & R@10 & R@5 & R@10 &
TIE$\downarrow$ & LCA$\downarrow$ &
$J\uparrow$ & $P_H\uparrow$ & $R_H\uparrow$\\

\midrule

CLIP & \xmark
&50.31&61.52&76.98&84.48&65.75&76.15&87.10&93.00
&4.43&2.51&0.72&0.79&0.80\\

CLIP & \cmark
&50.27&61.47&76.75&84.42&65.49&76.08&87.20&93.17
&4.42&2.50&0.72&0.78&0.80\\

SigLIP & \xmark
&52.32&63.90&77.54&85.24&65.98&76.78&\underline{88.60}&93.23
&4.22&2.47&0.75&0.83&0.82\\

SigLIP & \cmark
&52.27&63.81&77.21&84.09&\underline{66.03}&\underline{76.84}&\underline{88.60}&\underline{93.40}
&4.23&2.47&0.74&0.83&0.83\\

MERU & \xmark
&49.79&61.14&76.03&84.34&65.01&75.67&86.80&92.77
&4.44&2.52&0.71&0.78&0.79\\

MERU & \cmark
&49.63&61.08&76.13&84.21&65.16&75.74&86.93&92.77
&4.44&2.51&0.71&0.78&0.78\\

ATMG & \xmark
&50.95&62.36&77.21&84.74&64.96&75.81&87.10&93.17
&4.37&2.48&0.73&0.82&0.81\\

HyCoCLIP & \cmark
&51.35&63.13&77.92&85.52&64.84&75.78&87.30&93.23
&3.97&2.33&0.76&0.84&0.84\\

PHyCLIP & \cmark
&52.24&64.08&78.26&85.84&65.65&76.21&87.20&93.17
&3.92&2.31&0.76&\underline{0.84}&\underline{0.84}\\

ARGENT & \cmark
&\underline{52.68}&\underline{64.46}&\underline{78.59}&\underline{86.18}
&65.75&76.46&87.50&93.23
&\underline{3.91}&\underline{2.30}&\underline{0.77}&\underline{0.84}&\underline{0.84}\\

\rowcolor{oursRow}
\textbf{AURORA} & \cmark
&\textbf{55.77}&\textbf{66.81}&\textbf{81.02}&\textbf{88.18}
&\textbf{69.42}&\textbf{79.18}&\textbf{89.90}&\textbf{95.13}
&\textbf{3.62}&\textbf{2.25}&\textbf{0.78}&\textbf{0.85}&\textbf{0.85}\\

\bottomrule
\end{tabular}%
}
\end{table}

\vspace{-8pt}
\begin{table}[H]
\centering
\caption{\textbf{Hard-negative compositional classification on ViT-S/16.} Models are evaluated on VL-CheckList-Object and SugarCrepe. \textbf{Bold} indicates the best results. \underline{Underlining} denotes the second best.}

\label{tab:composition-main}
\vspace{-2pt}
\scriptsize
\setlength{\tabcolsep}{2.65pt}
\renewcommand{\arraystretch}{1.06}

\resizebox{\textwidth}{!}{%
\begin{tabular}{@{}*{16}{c}@{}}
\toprule

\multicolumn{2}{c}{\cellcolor{vitBand}\textbf{ViT-S/16}} &
\multicolumn{6}{c}{\cellcolor{generalHead}\textbf{VL-CheckList-Object}} &
\multicolumn{8}{c}{\cellcolor{fineHead}\textbf{SugarCrepe}}\\

\cmidrule(lr){3-8}
\cmidrule(lr){9-16}

\textbf{Model} & &
\multicolumn{3}{c}{\cellcolor{generalCell}\textbf{Location}} &
\multicolumn{3}{c}{\cellcolor{generalCell}\textbf{Size}} &
\multicolumn{3}{c}{\cellcolor{fineCell}\textbf{Replace}} &
\multicolumn{2}{c}{\cellcolor{fineCell}\textbf{Swap}} &
\multicolumn{3}{c}{\cellcolor{fineCell}\textbf{Add}}\\

& \raisebox{0pt}[0pt][0pt]{\shortstack[c]{w/\\boxes}} &
Center & Mid & Margin &
Large & Medium & Small &
Obj & Att & Rel &
Obj & Att &
Obj & Att & Overall\\

\midrule

CLIP & \xmark
&85.10&79.35&71.22&84.48&74.61&65.72
&86.74&78.00&69.04&59.49&65.27&77.53&70.76&75.71\\

CLIP & \cmark
&85.24&78.41&70.55&85.32&75.02&64.67
&85.94&77.45&\textbf{69.94}&59.89&64.41&77.90&70.76&75.69\\

SigLIP & \xmark
&89.33&83.04&74.91&88.30&78.50&68.82
&87.35&78.55&\underline{69.49}&\underline{60.82}&\underline{65.77}&78.08&71.10&76.26\\

SigLIP & \cmark
&89.83&83.03&74.50&88.13&78.13&69.47
&87.01&78.43&69.18&\textbf{61.38}&\textbf{66.67}&78.03&70.28&76.12\\

MERU & \xmark
&82.07&76.10&67.93&81.47&71.50&62.40
&86.40&77.58&68.73&58.94&64.96&77.17&70.33&75.35\\

MERU & \cmark
&82.90&75.40&68.40&81.33&71.73&62.40
&87.09&77.03&69.11&58.81&65.17&77.19&70.47&75.55\\

ATMG & \xmark
&84.07&78.40&69.83&83.53&73.60&64.53
&87.11&78.34&69.35&60.16&65.62&77.98&71.15&76.10\\

HyCoCLIP & \cmark
&90.80&86.75&79.95&90.52&82.73&74.23
&87.23&78.81&66.64&60.41&64.86&79.44&69.08&75.82\\

PHyCLIP & \cmark
&91.40&87.14&80.42&90.97&83.06&74.79
&88.40&80.25&67.33&59.76&65.22&79.89&71.48&76.71\\

ARGENT & \cmark
&\underline{91.87}&\underline{87.73}&\textbf{81.10}&\textbf{91.33}&\underline{84.07}&\underline{75.60}
&\underline{88.92}&\underline{80.88}&67.71&59.89&65.32&\underline{80.00}&\underline{72.59}&\underline{77.11}\\

\rowcolor{oursRow}
\textbf{AURORA} & \cmark
&\textbf{91.93}&\textbf{87.80}&\underline{80.77}&\underline{91.30}&\textbf{84.69}&\textbf{76.53}
&\textbf{89.77}&\textbf{81.85}&68.42&60.00&65.47&\textbf{80.16}&\textbf{74.27}&\textbf{77.75}\\

\bottomrule
\end{tabular}%
}
\end{table}

\vspace{-8pt}
\subsection{Geometric Visualizations}
\label{sec:visualizations}

We inspect AURORA's learned geometry from two complementary
perspectives: spatial-norm distributions reveal radial allocation
across four representation types, while
HoroPCA~\citep{chami2021horopca} provides a low-dimensional view
of their organization.
Figure~\ref{fig:geometry-visualization} compares AURORA ViT-S
with its w/o cone variant.
Without RAOC, full inputs and their corresponding boxes have
largely overlapping norm distributions within each modality.
RAOC produces more distinct distributions, particularly for
texts and text boxes, indicating stronger radial differentiation.
In the HoroPCA projections, the w/o cone variant exhibits
a narrow text configuration and a compact visual cluster.
With RAOC, both modalities occupy broader, more structured
regions, with text and image boxes tending to lie closer
to the projected origin than their full-input counterparts.
This projected arrangement is consistent with the intended
ordering from general concepts to contextualized instances.

\begin{figure}[t]
    \centering
    \input{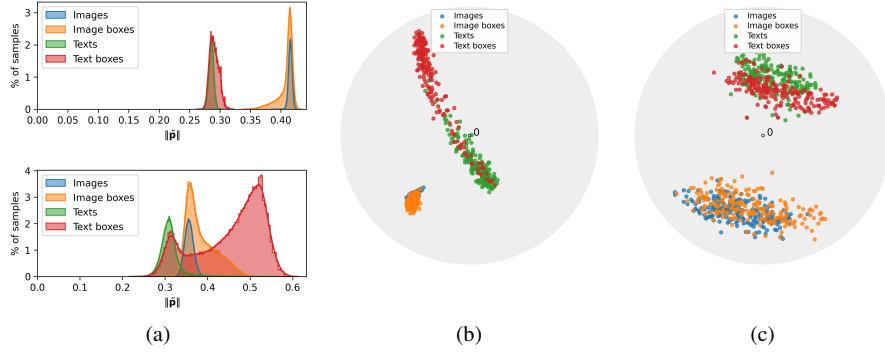}
    \caption{\textbf{Representation geometry with and without RAOC.}
    (a) Spatial-norm distributions and HoroPCA projections
    (b) without and (c) with RAOC for AURORA ViT-S.
    RAOC yields more distinct norm distributions and broader
    projected organization, with boxes tending to lie nearer
    the projected origin than full inputs.}
    \label{fig:geometry-visualization}
\end{figure}

We further examine local hierarchical organization through
radial gaps and root angles.
Figure~\ref{fig:hierarchy-visualization} compares AURORA ViT-S
with and without RAOC for the text query ``bag.''
With RAOC, the image-region tree forms distinct branches:
root angles decrease along the path from the full image
through ``mat bag'' to ``bag,'' while the umbrella--chair
branch remains less aligned with the query.
Without RAOC, the branches are less clearly distinguished.
These observations suggest that RAOC helps capture
local tree-like semantic structure.

\begin{figure}[t]
    \centering
    \input{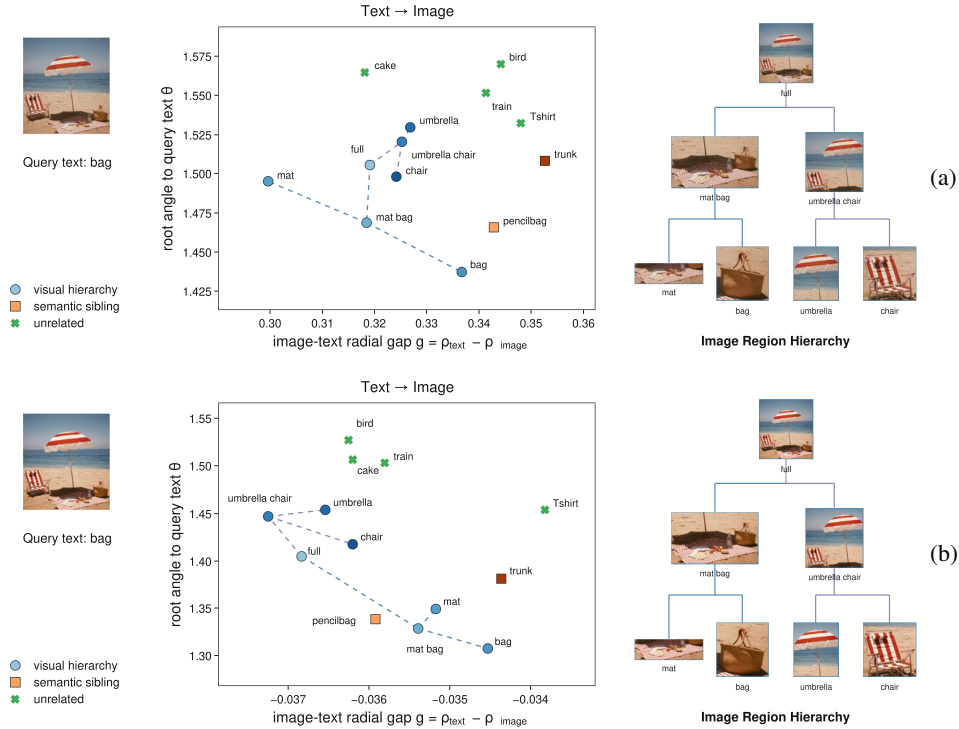}
    \caption{\textbf{Local hierarchical organization with and
    without RAOC.}
    AURORA ViT-S (a) w/o cone and (b) w/ cone for the text
    query ``bag.''
    Scatter plots show radial gaps and root angles for visual
    regions, semantic siblings, and unrelated controls;
    the corresponding visual hierarchy is shown on the right.
    RAOC produces clearer query-relative organization along
    the relevant hierarchy branch.}
    \label{fig:hierarchy-visualization}
    \vspace{-8pt}
\end{figure}

\vspace{-5pt}

\subsection{Ablation Study}
\label{sec:ablation}

Table~\ref{tab:ablation-main} isolates the main components of AURORA. Replacing RAOC with conventional Entailment Cone (EC) compares alternative order regions; removing the norm regularizer tests radial-scale control; removing RAOC isolates the effect of directed hierarchy supervision; and removing both RAOC and norm regularization leaves the root-angle alignment alone. Replacing the pairwise sigmoid loss with InfoNCE further examines the contrastive loss term under different alignment scores. With $-d_{\mathbb H}$, positive pairs are still encouraged toward smaller absolute distance, retaining the radial pressure discussed in Sec.~\ref{sec:analysis}. Appendix~\ref{app:sigmoid-lorentz} provides the corresponding gradient analysis.

\begin{table}[t]

\centering

\caption{\textbf{Ablation study of AURORA-ViT-S/16 pre-trained on grounded GRIT,} evaluated on image classification, COCO retrieval, and hierarchical classification. Controlled variants isolate the contribution of individual training components and loss terms.}

\label{tab:ablation-main}

\vspace{-8pt}

\scriptsize

\setlength{\tabcolsep}{3.2pt}

\renewcommand{\arraystretch}{1.06}

\resizebox{0.98\textwidth}{!}{%

\begin{tabular}{@{}c l l c l *{6}{c}@{}}

\toprule

\multicolumn{5}{c}{\cellcolor{vitBand}\textbf{AURORA-ViT-S/16}} &

\multicolumn{2}{c}{\cellcolor{generalHead}\textbf{Classification}} &

\multicolumn{2}{c}{\cellcolor{fineHead}\textbf{COCO R@5}} &

\multicolumn{2}{c}{\cellcolor{specialHead}\textbf{Hierarchy}}\\

\cmidrule(lr){6-7}\cmidrule(lr){8-9}\cmidrule(lr){10-11}

\textbf{ID} & \textbf{Alignment score} & \textbf{Order objective} & \textbf{Norm} & \textbf{Loss} &

\cellcolor{generalCell}ImageNet & \cellcolor{generalCell}Food-101 &

\cellcolor{fineCell}T$\rightarrow$I & \cellcolor{fineCell}I$\rightarrow$T &

\cellcolor{specialCell}TIE$\downarrow$ & \cellcolor{specialCell}$J\uparrow$\\

\midrule

(a) & root angle & EC & \cmark & Sigmoid
& 38.12 & 45.92 & 53.48 & 66.72 & 3.79 & 0.76\\

(b) & root angle & RAOC & \xmark & Sigmoid
& 40.21 & 48.57 & 54.98 & 68.41 & \underline{3.70} & \underline{0.77}\\

(c) & root angle & RAOC & \cmark & InfoNCE
& 39.42 & 47.71 & 54.43 & 67.93 & 3.76 & 0.76\\

(d) & root angle & none & \cmark & Sigmoid
& \underline{40.27} & 47.06 & \underline{55.29} & \underline{68.92} & 3.74 & \underline{0.77}\\

(e) & root angle & none & \xmark & Sigmoid
& 38.91 & 45.38 & 54.56 & 67.88 & 3.85 & 0.76\\

(f) & $-d_{\mathbb H}$ & EC & \cmark & Sigmoid
& 37.32 & 44.63 & 52.66 & 65.84 & 3.85 & 0.75\\

(g) & $-d_{\mathbb H}$ & EC & \cmark & InfoNCE
& 36.05 & 43.24 & 51.79 & 64.87 & 3.94 & 0.74\\

\midrule

\rowcolor{oursRow}
\textbf{Full} & root angle & RAOC & \cmark & Sigmoid
& \textbf{41.54} & \textbf{50.13}
& \textbf{55.77} & \textbf{69.42}
& \textbf{3.62} & \textbf{0.78}\\

\bottomrule

\end{tabular}}

\end{table}

\section{Related Work}
\label{sec:related}

\paragraph{Vision and language representation learning.}
CLIP~\citep{radford2021clip} and ALIGN~\citep{jia2021align} learn transferable visual and textual representations through contrastive learning on image and text pairs. SigLIP~\citep{zhai2023siglip} replaces softmax normalization with independent sigmoid losses over pairs. RegionCLIP~\citep{zhong2022regionclip} extends alignment to image regions and textual concepts. AURORA retains separate image and text encoders and examines how semantic alignment interacts with hierarchical supervision.

\paragraph{Hyperbolic vision and language representations.}
Hyperbolic embeddings capture hierarchical structure in language and vision~\citep{nickel2017poincare,khrulkov2020hyperbolic}. MERU combines Lorentz representations with entailment constraints~\citep{desai2023meru}. HyCoCLIP incorporates object regions and grounded phrases through compositional contrastive and entailment learning~\citep{pal2025hycoclip}. PHyCLIP uses an $\ell_1$ product of hyperbolic factors to model hierarchy and composition~\citep{yoshikawa2026phyclip}. Wu et al.~\citep{wu2026alignment} align visual and textual trees across hyperbolic manifolds with different curvatures. ATMG addresses the tension between distance alignment and entailment through an exterior angle contrastive objective~\citep{ramasinghe2024accept}. AURORA separates root angle alignment from radial ordering to preserve differences in specificity.

\paragraph{Hierarchical representations and entailment regions.}
Hierarchical embeddings encode asymmetric relations through coordinatewise partial orders~\citep{vendrov2016order}. Hyperbolic entailment cones define nested descendant regions~\citep{ganea2018cones} and support concept hierarchy induction from text~\citep{le2019inferring}. Shadow Cones~\citep{yu2024shadow} generalizes entailment regions through umbral and penumbral constructions based on geometric occlusion. ARGENT~\citep{huynh2026argent} combines adaptive entailment supervision with norm regularization to stabilize hierarchical representations. AURORA defines descendant membership through relative radial separation and root angle difference, with contrastive supervision over matched and unmatched pairs.

\section{Conclusion}

We introduced AURORA, a hyperbolic vision--language pretraining framework that separates angular semantic alignment from radial hierarchical ordering. Our analysis shows that distance minimization can conflict with hierarchical separation, while semantic discrimination alone does not ensure valid entailment. AURORA combines angular alignment with the Radial--Angular Order Cone, which defines a partial order through radial gaps and angular differences. Hierarchical entailment contrastive training reinforces annotated relations and suppresses spurious entailment, while norm regularization stabilizes training. Experiments demonstrate improvements over competitive Euclidean and hyperbolic baselines across classification, retrieval, hierarchical classification, and compositional understanding. Future work could discover richer hierarchical relations within and across image--text examples and explore how to learn from them with fewer explicit annotations.
\subsection*{AI Use Statement}

Generative AI tools were used for literature retrieval, language editing,
LaTeX assistance, and discussions of research ideas and experimental design.
They were also used as a supporting tool during the verification and
organization of mathematical derivations. All formulations,
implementations, experiments, proofs, and scientific claims were developed,
verified, and validated by the authors, who take full responsibility for
the final content of this manuscript.

\subsection*{Reproducibility Statement}

The main paper presents the objectives, methodology, and key experimental
findings. Appendices~\ref{app:implementation}, \ref{app:protocols},
and related supplementary sections provide implementation details,
training configurations, evaluation protocols, mathematical derivations,
and computational settings. We will release code, configurations, data
processing details, evaluation resources, and experiment outputs to
facilitate reproduction.
\bibliography{references}
\bibliographystyle{iclr2027_conference}

\clearpage
\appendix
\section{Notation and Equivalent Geometric Views}
\label{app:notation}

\begin{table}[H]
\centering
\caption{Core notation used in the geometric analysis.}
\label{tab:notation}
\small
\resizebox{\textwidth}{!}{%
\begin{tabular}{lll}
\toprule
Symbol & Meaning & Role \\
\midrule
$c$ & positive curvature magnitude; manifold curvature is $-c$ & learned global scalar \\
$\tilde{\mathbf{x}}\in\mathbb R^d$ & spatial component of a Lorentz point $\mathbf{x}$ & Lorentz coordinate \\
$r_x=\|\tilde{\mathbf{x}}\|$ & Lorentz spatial norm & radial regularization variable \\
$\rho_x=\sqrt c\,d_c(O,\mathbf{x})$ & scaled distance from the origin & radial order coordinate \\
$\mathbf{u}_x=\tilde{\mathbf{x}}/\|\tilde{\mathbf{x}}\|$ & spatial direction of non-origin $\mathbf{x}$ & angular coordinate \\
$\mathbf{w}_x=\Phi_c(\mathbf{x})$ & Poincar\'e coordinate of $\mathbf{x}$ & equivalent coordinate view \\
$\theta_{xy}=\angle \mathbf{x}O\mathbf{y}$ & root angle & directional difference \\
$g(\mathbf{p},\mathbf{q})=\rho_q-\rho_p$ & parent--candidate radial gap & radial precedence \\
$f_\beta(g)=\min\{\pi,\beta g\}$ & admissible angle for $g\ge0$ & RAOC aperture \\
$m_r,m_a$ & bounded radial and angular margins & optimization margins \\
$D_{\mathrm{RAOC}}=\min\{m_r,m_a\}$ & bounded entailment compatibility & training score \\
$\beta$ & radial-gap-to-angle slope & order geometry \\
$\sigma_r$ & radial margin scale & score scaling \\
$a_e,b_e$ & entailment-logit scale and bias & score calibration \\
\bottomrule
\end{tabular}}
\vspace{-5pt}
\end{table}

\subsection{Radial--angular coordinates and root-angle identities}
\label{app:rootangle}
\label{app:model-conventions}

Throughout the appendix we use the radial--angular representation induced by the Lorentz origin. We write Lorentz vectors as $\mathbf{x}=[x_0;\tilde{\mathbf{x}}]\in\mathbb R^{d+1}$, where $x_0\in\mathbb R$ and $\tilde{\mathbf{x}}\in\mathbb R^d$; the semicolon denotes vertical concatenation.  For any non-origin $\mathbf{x}\in\Hc$, the manifold constraint together with $\rho_x=\sqrt c\,d_c(O,\mathbf{x})$ gives
\begin{equation}
 \mathbf{x}=
 \frac{1}{\sqrt c}
 \bigl[\cosh\rho_x;\,\sinh\rho_x\,\mathbf{u}_x\bigr],
 \qquad
 \mathbf{u}_x=\frac{\tilde{\mathbf{x}}}{\|\tilde{\mathbf{x}}\|}\in\mathbb S^{d-1}.
\label{eq:app-polar}
\end{equation}
Equivalently,
\begin{equation}
 r_x=\frac{\sinh\rho_x}{\sqrt c},
 \qquad
 \rho_x=\operatorname{asinh}(\sqrt c\,r_x).
\label{eq:app-radial-conversion}
\end{equation}
For the origin-exponential parameterization in Eq.~\eqref{eq:expmap-origin}, this simplifies to
$\rho_x=\sqrt c\,\|\mathbf{z}_x\|$ and $\mathbf{u}_x=\mathbf{z}_x/\|\mathbf{z}_x\|$ whenever $\mathbf{z}_x\neq\mathbf{0}_d$.
Consequently, for non-origin $\mathbf{x}$ and $\mathbf{y}$,
\begin{equation}
\cos\theta_{xy}
=
 \mathbf{u}_x^\top \mathbf{u}_y
=
\frac{\tilde{\mathbf{x}}^\top \tilde{\mathbf{y}}}{\|\tilde{\mathbf{x}}\|\,\|\tilde{\mathbf{y}}\|}
=
\frac{\cosh\rho_x\cosh\rho_y-\cosh\delta_{xy}}
{\sinh\rho_x\sinh\rho_y},
\qquad
\delta_{xy}=\sqrt c\,d_c(\mathbf{x},\mathbf{y}).
\label{eq:app-root-identities}
\end{equation}
The first equality gives the score used by AURORA, and the last follows directly from the hyperbolic cosine law.  The resulting semantic score is invariant to radial position and varies with root direction.

\subsection{Poincar\'e realization of root-angle alignment and RAOC}
\label{app:poincare-equivalence}

A natural question is how the same geometry appears when the embeddings are parameterized directly in the Poincar\'e ball.  For curvature $-c$,
\begin{equation}
\mathbb B_c^d
=
\{\mathbf{w}\in\mathbb R^d:c\|\mathbf{w}\|^2<1\},
\qquad
g_w^{\mathbb B}
=
\left(\frac{2}{1-c\|\mathbf{w}\|^2}\right)^2 I .
\label{eq:app-poincare-metric}
\end{equation}
The metric is conformal to the Euclidean metric, so Riemannian angles and Euclidean coordinate angles agree at each ball point~\citep{nickel2017poincare,ganea2018cones}.  The origin-preserving Lorentz--Poincar\'e isometry is
\begin{equation}
\Phi_c(\mathbf{x})
=
\frac{\tilde{\mathbf{x}}}{\sqrt c\,x_0+1}.
\label{eq:app-lorentz-to-poincare}
\end{equation}
Substituting Eq.~\eqref{eq:app-polar} gives
\begin{equation}
\mathbf{w}_x=\Phi_c(\mathbf{x})
=
\frac{1}{\sqrt c}\tanh\!\left(\frac{\rho_x}{2}\right)\mathbf{u}_x,
\qquad
\rho_x
=
2\operatorname{artanh}(\sqrt c\,\|\mathbf{w}_x\|),
\qquad
\frac{\mathbf{w}_x}{\|\mathbf{w}_x\|}=\mathbf{u}_x.
\label{eq:app-poincare-polar}
\end{equation}

\paragraph{Root-angle alignment.}
The root angle is the intrinsic angle between the two radial geodesics leaving the origin.  The isometry $\Phi_c$ preserves this angle and maps the Lorentz origin $O$ to the ball origin $0$.  In the Poincar\'e ball, radial geodesics through $0$ are Euclidean radial segments; conformality then identifies their intrinsic angle with the ordinary Euclidean polar angle.  For corresponding non-origin points $\mathbf{w}_x=\Phi_c(\mathbf{x})$ and $\mathbf{w}_y=\Phi_c(\mathbf{y})$,
\begin{equation}
\begin{aligned}
\theta_{xy}^{\mathbb L}
&=
\angle_O^{\mathbb L}(\mathbf{x},\mathbf{y})
=
\angle_0^{\mathbb B}(\mathbf{w}_x,\mathbf{w}_y) \\
&=
\arccos\!\left(
\frac{\mathbf{w}_x^\top \mathbf{w}_y}{\|\mathbf{w}_x\|\,\|\mathbf{w}_y\|}
\right),
\end{aligned}
\label{eq:app-root-angle-model-equivalence}
\end{equation}
and therefore
\begin{equation}
s_{\angle}^{\mathbb L}(\mathbf{x},\mathbf{y})
=
\frac{\tilde{\mathbf{x}}^\top \tilde{\mathbf{y}}}{\|\tilde{\mathbf{x}}\|\,\|\tilde{\mathbf{y}}\|}
=
\frac{\mathbf{w}_x^\top \mathbf{w}_y}{\|\mathbf{w}_x\|\,\|\mathbf{w}_y\|}
=
s_{\angle}^{\mathbb B}(\mathbf{w}_x,\mathbf{w}_y).
\label{eq:app-root-poincare-equivalence}
\end{equation}
A Poincar\'e implementation can consequently optimize root-angle alignment with the Euclidean cosine between normalized ball coordinates.  The Lorentz spatial-direction cosine and the Poincar\'e polar-angle cosine are two coordinate expressions of the same intrinsic root angle.

\paragraph{RAOC in Poincar\'e coordinates.}
The isometry also preserves distance from the origin.  Define the scaled Poincar\'e radial depth
\begin{equation}
\rho_{\mathbb B}(\mathbf{w})
=
\sqrt c\,d_{\mathbb B,c}(0,\mathbf{w})
=
2\operatorname{artanh}(\sqrt c\,\|\mathbf{w}\|).
\label{eq:app-poincare-radius}
\end{equation}
Then $\rho_{\mathbb B}(\mathbf{w}_x)=\rho_x$, and Eq.~\eqref{eq:app-raoc-set} has the equivalent ball form
\begin{equation}
\begin{aligned}
\mathbf{q}\in\mathcal C_\beta(\mathbf{p})
\quad\Longleftrightarrow\quad
&\rho_{\mathbb B}(\mathbf{w}_q)-\rho_{\mathbb B}(\mathbf{w}_p)\ge0,\\
&\theta_{\mathbb B}(\mathbf{w}_p,\mathbf{w}_q)
\le
\beta\bigl[\rho_{\mathbb B}(\mathbf{w}_q)-\rho_{\mathbb B}(\mathbf{w}_p)\bigr],
\end{aligned}
\label{eq:app-raoc-poincare}
\end{equation}
where $\theta_{\mathbb B}$ is the Euclidean polar angle from Eq.~\eqref{eq:app-root-angle-model-equivalence}.  The order is therefore expressed by the same radial gap and the same direction-sphere angle in either hyperbolic model.  The partial-order and descendant-set nesting results below carry directly to the Poincar\'e representation.

For a fixed parent $\mathbf{w}_p$, Eq.~\eqref{eq:app-raoc-poincare} can also be written in the Euclidean ball radius:
\begin{equation}
\|\mathbf{w}_q\|
\ge
\frac{1}{\sqrt c}
\tanh\!\left[
\operatorname{artanh}(\sqrt c\,\|\mathbf{w}_p\|)
+
\frac{\theta_{\mathbb B}(\mathbf{w}_p,\mathbf{w}_q)}{2\beta}
\right].
\label{eq:app-raoc-poincare-boundary}
\end{equation}
The linear RAOC boundary in $(g,\theta)$ coordinates thus appears as a curved radial--angular contour when expressed with the Euclidean ball radius.  The angular coordinate remains the Euclidean polar angle, while radial order is measured by the intrinsic depth $\rho_{\mathbb B}$.

When an encoder feature is inserted through the origin exponential map, the tangent-coordinate conventions differ by a constant scale.  Since $g_0^{\mathbb B}=4I$,
\begin{equation}
(\mathrm d\Phi_c)_O\!\left([0;\mathbf{z}]\right)=\mathbf{z}/2,
\qquad
\Phi_c\!\left(\exp_O^c\!\left([0;\mathbf{z}]\right)\right)
=
\frac{\tanh(\sqrt c\|\mathbf{z}\|/2)}{\sqrt c}\frac{\mathbf{z}}{\|\mathbf{z}\|}.
\label{eq:app-origin-coordinate-scale}
\end{equation}
Thus the Lorentz tangent coordinate $\mathbf{z}$ and the Poincar\'e coordinate tangent vector $\mathbf{z}/2$ parameterize the corresponding manifold point with the same normalized direction.  Coordinate-wise optimization trajectories can vary with parameterization and optimizer choice.  At the geometric level, corresponding manifold configurations receive identical root-angle scores and RAOC membership decisions under $\Phi_c$.

\section{Geometric Analysis of Alignment and Entailment}
\label{app:proofs}

\subsection{Distance alignment and conventional cone geometry}
\label{app:alignment-analysis}

\paragraph{Radial pressure from pairwise distance.}
Let $\delta_{pq}=\sqrt c\,d_c(\mathbf{p},\mathbf{q})$.
Starting from Eq.~\eqref{eq:lawcos}, fix $\rho_p>0$
and $\theta_{pq}$, with $\mathbf{p}\neq \mathbf{q}$ and $\rho_q\ge\rho_p$.
Differentiating with respect to $\rho_q$ gives
\begin{equation}
\sinh\delta_{pq}\,
\frac{\partial\delta_{pq}}{\partial\rho_q}
=
\cosh\rho_p\cosh\rho_q
\left(
\tanh\rho_q-\tanh\rho_p\cos\theta_{pq}
\right).
\label{eq:app-distance-derivative}
\end{equation}
The factor in parentheses can be decomposed as
\[
(\tanh\rho_q-\tanh\rho_p)
+
\tanh\rho_p(1-\cos\theta_{pq}).
\]
Both terms are nonnegative, and simultaneous equality would
imply $\rho_q=\rho_p$ and $\theta_{pq}=0$, hence $\mathbf{p}=\mathbf{q}$.
Since $\delta_{pq}>0$, it follows that
$\partial\delta_{pq}/\partial\rho_q>0$.
Holding the parent and both directions fixed, decreasing the
outward child's radius toward the parent's therefore decreases
their distance.

Moreover, as their directions align,
\begin{equation}
\lim_{\theta_{pq}\to0}d_c(\mathbf{p},\mathbf{q})
=
\frac{\rho_q-\rho_p}{\sqrt c},
\qquad \rho_q>\rho_p.
\label{eq:aligned-distance-limit}
\end{equation}
Thus, even at zero angular separation, distance minimization
continues to penalize the parent--child radial gap.

\paragraph{Pairwise distance does not determine cone membership.}
Let $\mathcal C_{\mathrm{EC}}(\mathbf{p})$ denote an outward
conventional Entailment Cone (EC) rooted at a non-origin parent $\mathbf{p}$.
Its half-aperture $\omega(\mathbf{p})$ specifies the maximum permitted
angle from the outward continuation of the radial geodesic $O\mathbf{p}$.
For $\mathbf{q}\neq \mathbf{p}$, let $\varphi(\mathbf{p},\mathbf{q})$ be the angle at $\mathbf{p}$
between this outward direction and the geodesic from $\mathbf{p}$ to $\mathbf{q}$.
Then $\mathbf{q}\in\mathcal C_{\mathrm{EC}}(\mathbf{p})$ if and only if
$\varphi(\mathbf{p},\mathbf{q})\le\omega(\mathbf{p})$.

\begin{proposition}[Distance and cone membership]
\label{prop:distance-membership}
Suppose $\mathcal C_{\mathrm{EC}}(\mathbf{p})$ has half-aperture
$\omega(\mathbf{p})\in(0,\pi/2]$.
For distinct non-origin embeddings $\mathbf{p}$ and $\mathbf{q}$:
\begin{enumerate}[leftmargin=1.4em,itemsep=0pt,topsep=2pt]
    \item If $\rho_p=\rho_q$ and $\theta_{pq}>0$, then
    $\varphi(\mathbf{p},\mathbf{q})>\pi/2$ and
    $\mathbf{q}\notin\mathcal C_{\mathrm{EC}}(\mathbf{p})$,
    even though $d_c(\mathbf{p},\mathbf{q})\to0$ as $\theta_{pq}\to0$.
    \item If $\theta_{pq}=0$ and $\rho_q>\rho_p$, then
    $\varphi(\mathbf{p},\mathbf{q})=0$ and
    $\mathbf{q}\in\mathcal C_{\mathrm{EC}}(\mathbf{p})$,
    while $d_c(\mathbf{p},\mathbf{q})=(\rho_q-\rho_p)/\sqrt c$
    grows without bound as $\rho_q$ increases.
\end{enumerate}
\end{proposition}

\begin{proof}
The hyperbolic cosine law at $\mathbf{p}$ gives
\begin{equation}
\cos\varphi(\mathbf{p},\mathbf{q})
=
\frac{
\cosh\rho_q-
\cosh\rho_p\cosh\delta_{pq}
}{
\sinh\rho_p\sinh\delta_{pq}
}.
\label{eq:appendix-exterior-angle}
\end{equation}
If $\rho_p=\rho_q=\rho>0$ and $\theta_{pq}>0$,
Eq.~\eqref{eq:lawcos} gives
\begin{equation}
\cosh\delta_{pq}
=
1+\sinh^2\rho\,(1-\cos\theta_{pq})>1.
\end{equation}
The numerator of Eq.~\eqref{eq:appendix-exterior-angle}
is therefore negative and the denominator is positive,
so $\varphi(\mathbf{p},\mathbf{q})>\pi/2$.
Any outward cone with half-aperture in $(0,\pi/2]$
excludes $\mathbf{q}$, although $d_c(\mathbf{p},\mathbf{q})\to0$ as
$\theta_{pq}\downarrow0$.

Conversely, if $\theta_{pq}=0$ and $\rho_q>\rho_p$,
then $\mathbf{p}$ and $\mathbf{q}$ lie on the same outward radial geodesic.
Hence $\delta_{pq}=\rho_q-\rho_p$ and $\varphi(\mathbf{p},\mathbf{q})=0$.
The point $\mathbf{q}$ remains inside the cone while
$d_c(\mathbf{p},\mathbf{q})=(\rho_q-\rho_p)/\sqrt c$ can grow arbitrarily large.
\end{proof}

Proposition~\ref{prop:distance-membership} shows that
distance attraction does not guarantee cone membership,
and distance repulsion does not guarantee exclusion.
In particular, an unrelated candidate can move farther
outward along the parent's radial ray without leaving
its descendant region.
Positive-only entailment supervision provides no direct
penalty for this erroneous relation.

\paragraph{Why parent-radius apertures can lose selectivity.}
For conventional Lorentz entailment cones~\citep{ganea2018cones,desai2023meru},
the half-aperture is
\begin{equation}
\omega(\mathbf{p})
=\arcsin\!\left(\frac{2\kappa}{\sqrt c\,\|\tilde{\mathbf{p}}\|}\right)
=\arcsin\!\left(\frac{2\kappa}{\sinh\rho_p}\right),
\qquad \kappa>0.
\label{eq:hec-aperture}
\end{equation}
Here $\kappa$ is a fixed constant controlling cone width, and
the admissible domain is
$\rho_p\ge\operatorname{arsinh}(2\kappa)$.
Thus, $\sin\omega(\mathbf{p})=2\kappa/\sinh\rho_p$, so directional
selectivity is controlled by the \emph{absolute} parent radius.
As $\rho_p$ decreases toward the admissible lower bound,
$\omega(\mathbf{p})$ approaches $\pi/2$: the permitted directions
fill the outward hemisphere, and the corresponding geodesic
cone approaches an outward half-space.
The same phenomenon is visible in Poincar\'e coordinates:
from Eq.~\eqref{eq:app-poincare-polar},
\begin{equation}
\sinh\rho_p
=
\frac{2\sqrt c\,\|\mathbf{w}_p\|}{1-c\|\mathbf{w}_p\|^2},
\end{equation}
so the identical aperture becomes
\begin{equation}
\omega(\mathbf{w}_p)
=
\arcsin\!\left(
\kappa\,
\frac{1-c\|\mathbf{w}_p\|^2}{\sqrt c\,\|\mathbf{w}_p\|}
\right).
\label{eq:app-hec-poincare}
\end{equation}
The absolute-depth coupling between cone width and parent position
is therefore identical in the two coordinate models.

This limiting behavior also appears at practical operating points.
ARGENT identifies parent-norm contraction and half-space widening
as a cone-collapse failure mode in hyperbolic VLM
training~\citep{huynh2026argent}.
More recently, an operating-point audit of released MERU,
HyCoCLIP, and PHyCLIP models reports regimes in which conventional
cone or traversal mechanisms are saturated, inactive, or weakly
aligned with the intended hierarchy~\citep{kim2026operating}.
These observations motivate an order geometry whose angular
selectivity is governed by the \emph{relative parent--child radial gap}.
RAOC couples this gap to root-angle separation, and contrastive
entailment supplies supervision for matched and unmatched relations.

\section{RAOC as an Order Geometry}
\label{app:raoc-proof}

RAOC assigns each non-origin parent $\mathbf{p}$ the descendant set
\begin{equation}
\mathcal C_\beta(\mathbf{p})
=
\left\{
\mathbf{q}\in\Hc\setminus\{O\}:
\rho_q-\rho_p\ge\frac{\theta_{pq}}{\beta}
\right\}.
\label{eq:app-raoc-set}
\end{equation}

This form gives a direct semantic interpretation.  Consider the parent description ``\textit{a dog}'' and the more specific description ``\textit{a dog running on grass}.''  Their close directions allow a modest outward radial gap to encode the increase in specificity.  An unrelated description such as ``\textit{a rocket successfully lifting off}'' typically forms a much larger root angle, which raises the radial depth required to enter the dog's descendant region.  Contrastive entailment supplies the semantic signal that separates these matched and unmatched relations.  Once a parent--child relation is formed, the geometric rule composes consistently along deeper chains.

\begin{theorem}[RAOC partial order]
\label{thm:partial-order}
For any fixed $\beta>0$, define $\mathbf{p}\preceq_\beta \mathbf{q}$ if and only if $\mathbf{q}\in\mathcal C_\beta(\mathbf{p})$.  Then $\preceq_\beta$ is reflexive, antisymmetric, and transitive on $\Hc\setminus\{O\}$.  Moreover,
\begin{equation}
 \mathbf{p}\preceq_\beta \mathbf{q}
 \quad\Longleftrightarrow\quad
 \mathcal C_\beta(\mathbf{q})\subseteq\mathcal C_\beta(\mathbf{p}).
\label{eq:raoc-nesting}
\end{equation}
\end{theorem}

\begin{proof}
By definition,
\begin{equation}
 \mathbf{p}\preceq_\beta \mathbf{q}
 \quad\Longleftrightarrow\quad
 \rho_q-\rho_p\ge\theta_{pq}/\beta.
\label{eq:app-raoc-equivalent}
\end{equation}
Reflexivity follows because both sides vanish for $\mathbf{p}=\mathbf{q}$.

For antisymmetry, suppose $\mathbf{p}\preceq_\beta \mathbf{q}$ and $\mathbf{q}\preceq_\beta \mathbf{p}$.  Adding the two inequalities gives
$0\ge2\theta_{pq}/\beta$, hence $\theta_{pq}=0$.  The two radial inequalities then force $\rho_p=\rho_q$.  Equation~\eqref{eq:app-polar} gives $\mathbf{p}=\mathbf{q}$.

For transitivity, suppose $\mathbf{p}\preceq_\beta \mathbf{q}$ and $\mathbf{q}\preceq_\beta \mathbf{r}$.  The spherical angle is a metric on the direction sphere, so
\begin{equation}
\theta_{pr}
\le
\theta_{pq}+\theta_{qr}
\le
\beta(\rho_q-\rho_p)+\beta(\rho_r-\rho_q)
=
\beta(\rho_r-\rho_p).
\end{equation}
Hence $\mathbf{p}\preceq_\beta \mathbf{r}$.

Finally, if $\mathbf{p}\preceq_\beta \mathbf{q}$, transitivity places every descendant of $\mathbf{q}$ inside $\mathcal C_\beta(\mathbf{p})$, proving the forward implication in Eq.~\eqref{eq:raoc-nesting}.  Conversely, reflexivity gives $\mathbf{q}\in\mathcal C_\beta(\mathbf{q})$; therefore $\mathcal C_\beta(\mathbf{q})\subseteq\mathcal C_\beta(\mathbf{p})$ implies $\mathbf{q}\in\mathcal C_\beta(\mathbf{p})$.  This proves the reverse implication.
\end{proof}

If $\mathbf{p}\preceq_\beta \mathbf{q}$ and $\mathbf{p}\neq \mathbf{q}$, then the inclusion is strict.  For example, if ``\textit{a dog}'' precedes ``\textit{a dog running on grass},'' every representation that is a descendant of the latter is automatically a descendant of the former.  This nested-region property expresses increasing specificity through a shared relative-gap rule across radial levels.

\subsection{Why the capped linear aperture is a natural transitive choice}
\label{app:order-derivation}

More generally, consider the radial-gap aperture family
\begin{equation}
 \mathbf{p}\preceq_f \mathbf{q}
 \quad\Longleftrightarrow\quad
 g(\mathbf{p},\mathbf{q})\ge0,
 \qquad
 \theta_{pq}\le f(g(\mathbf{p},\mathbf{q})),
\end{equation}
where $f:[0,\infty)\to[0,\pi]$ and $f(0)=0$.  The next result characterizes exactly when such a construction is transitive over the full direction sphere.

\begin{theorem}[Transitive radial-gap apertures]
\label{thm:aperture-characterization}
Let $d\ge2$.  The relation $\preceq_f$ above is transitive on $\Hc\setminus\{O\}$ if and only if
\begin{equation}
 f(a+b)
 \ge
 \min\{\pi,f(a)+f(b)\},
 \qquad a,b\ge0.
\label{eq:app-fcondition}
\end{equation}
\end{theorem}

\begin{proof}
For sufficiency, let $\mathbf{p}\preceq_f \mathbf{q}$ and $\mathbf{q}\preceq_f \mathbf{r}$, and set $a=g(\mathbf{p},\mathbf{q})$ and $b=g(\mathbf{q},\mathbf{r})$.  Radial gaps add, while the spherical triangle inequality gives
\begin{equation}
\theta_{pr}
\le
\min\{\pi,\theta_{pq}+\theta_{qr}\}
\le
\min\{\pi,f(a)+f(b)\}
\le
f(a+b).
\end{equation}
Thus $\mathbf{p}\preceq_f \mathbf{r}$.

For necessity, fix $a,b\ge0$ and a base radius $\rho_0>0$.  Choose radii $\rho_p=\rho_0$, $\rho_q=\rho_0+a$, and $\rho_r=\rho_0+a+b$.  Since $d\ge2$, choose their directions on a common great circle with consecutive increments
\begin{equation}
\eta_1=f(a),
\qquad
\eta_2=\min\{f(b),\pi-f(a)\}.
\end{equation}
Then $\mathbf{p}\preceq_f \mathbf{q}$ and $\mathbf{q}\preceq_f \mathbf{r}$, while
\begin{equation}
\theta_{pr}
=
\eta_1+\eta_2
=
\min\{\pi,f(a)+f(b)\}.
\end{equation}
Transitivity therefore requires Eq.~\eqref{eq:app-fcondition}.
\end{proof}

\begin{corollary}[Pointwise minimality at fixed local slope]
\label{cor:minimal-aperture}
Suppose $\preceq_f$ is transitive and
\begin{equation}
\lim_{g\downarrow0}\frac{f(g)}{g}=\beta>0.
\end{equation}
Then
\begin{equation}
 f(g)\ge\min\{\pi,\beta g\}
 =f_\beta(g)
 \qquad\text{for all }g\ge0.
\label{eq:minimal-aperture}
\end{equation}
\end{corollary}

\begin{proof}
For $g>0$, repeated application of Eq.~\eqref{eq:app-fcondition} gives
$f(g)\ge\min\{\pi,nf(g/n)\}$ for every positive integer $n$.  Since $nf(g/n)\to\beta g$, taking $n\to\infty$ yields Eq.~\eqref{eq:minimal-aperture}.  The case $g=0$ follows from $f(0)=0$.
\end{proof}

The capped linear rule $f_\beta(g)=\min\{\pi,\beta g\}$ satisfies Eq.~\eqref{eq:app-fcondition} with equality.  At fixed local slope $\beta$, it is the pointwise tightest globally transitive aperture in this family.  Equation~\eqref{eq:raoc-region} therefore selects the smallest descendant region compatible with the same local expansion rate and global transitivity.

\subsection{Multiple parents and radial depth}
\label{app:composition-geometry}

RAOC also gives a simple geometric interpretation of compositional constraints.  For non-origin parents $\mathbf{p}_1,\ldots,\mathbf{p}_k$ and a fixed candidate direction $\mathbf{u}$, simultaneous membership is equivalent to
\begin{equation}
\rho_q
\ge
\rho_{\min}(\mathbf{u})
=
\max_i
\left\{
\rho_{p_i}
+
\frac{\theta(\mathbf{u}_{p_i},\mathbf{u})}{\beta}
\right\},
\qquad
\theta(\mathbf{u},\mathbf{v})=\arccos(\mathbf{u}^\top \mathbf{v}).
\label{eq:multiple-parent-depth}
\end{equation}
Thus each parent contributes a lower bound on the radial depth required along the chosen direction, and the strictest parent determines the feasible depth.

For two parents with the same radius $\rho$ and angular separation $\Theta$, with $d\ge2$,
\begin{equation}
\min_{\mathbf{q}\in\mathcal C_\beta(\mathbf{p}_1)\cap\mathcal C_\beta(\mathbf{p}_2)}\rho_q
=
\rho+\frac{\Theta}{2\beta}.
\label{eq:two-parent-depth}
\end{equation}
Indeed, the triangle inequality implies
\begin{equation}
\Theta
\le
\theta(\mathbf{u}_{p_1},\mathbf{u})+\theta(\mathbf{u},\mathbf{u}_{p_2})
\le
2\max_i\theta(\mathbf{u}_{p_i},\mathbf{u}),
\end{equation}
which gives the lower bound, and the midpoint of a shortest spherical path between the two parent directions attains it.

For a chain $\mathbf{p}_0\preceq_\beta\cdots\preceq_\beta \mathbf{p}_k$, summing the edge constraints yields
\begin{equation}
\sum_{i=0}^{k-1}\theta_{p_ip_{i+1}}
\le
\beta(\rho_{p_k}-\rho_{p_0}).
\label{eq:chain-budget}
\end{equation}
The total radial depth therefore acts as a geometric budget for accumulated directional change along a hierarchy.

\subsection{Geodesic convexity versus order consistency}
\label{app:raoc-convexity}

RAOC descendant regions can be non-convex along hyperbolic geodesics.  A compact counterexample is obtained by fixing a parent direction $\mathbf{u}_p$, choosing $\mathbf{v}\perp \mathbf{u}_p$, and taking
\begin{equation}
 \mathbf{u}_\pm=\cos a\,\mathbf{u}_p\pm\sin a\,\mathbf{v},
 \qquad 0<a<\pi/2.
\end{equation}
Place two descendants in these directions at a common radius $R\ge\rho_p+\pi/\beta$.  Their Lorentz geodesic midpoint lies in direction $\mathbf{u}_p$ and has radius $\rho_m$ satisfying
\begin{equation}
\tanh\rho_m=\tanh R\cos a.
\end{equation}
For $a$ sufficiently close to $\pi/2$, $\rho_m<\rho_p$, and the midpoint falls outside the descendant region.

For hierarchical entailment, transitivity and nested descendant sets govern compositional consistency.  Geodesic convexity controls the shape of the region along geodesics, while angular selectivity controls which directions are admitted at a given hierarchical depth.  The conventional EC analysis above shows that a geodesically well-structured cone can still approach a broad half-space as the parent contracts, substantially weakening directional discrimination~\citep{huynh2026argent,kim2026operating}.  RAOC makes this selectivity explicit through parent--child radial difference and root-angle separation.

\section{Optimization Analysis}
\label{app:optimization}

\subsection{Absolute-distance calibration and radial pressure}
\label{app:sigmoid-lorentz}

Let $\Delta=d_c(\mathbf{x},\mathbf{y})\ge0$ be the unscaled geodesic distance and use the score $s=-\Delta$.  With sigmoid logit $\zeta=b-a\Delta$, $a>0$, the positive- and negative-pair losses are
\begin{equation}
\ell_+(\Delta)=\log(1+\exp(a\Delta-b)),
\qquad
\ell_-(\Delta)=\log(1+\exp(b-a\Delta)).
\end{equation}
Their derivatives are
\begin{equation}
\frac{\partial\ell_+}{\partial\Delta}
=
 a\,\sigma(a\Delta-b)>0,
\qquad
\frac{\partial\ell_-}{\partial\Delta}
=
-a\,\sigma(b-a\Delta)<0.
\label{eq:sig-distance-gradients}
\end{equation}
Thus every positive distance term favors a smaller absolute distance, whereas every negative term favors a larger one.  Combined with Eq.~\eqref{eq:app-distance-derivative}, the positive term exerts inward radial pressure on the outward point under the fixed-direction conditions analyzed above.

For comparison, row-wise InfoNCE with score $-a\Delta$ is
\begin{equation}
\ell_i^{\mathrm{NCE}}
=
 a\Delta_{ii}
+
\log\sum_j\exp(-a\Delta_{ij}).
\end{equation}
If
\begin{equation}
P_{ij}
=
\frac{\exp(-a\Delta_{ij})}{\sum_k\exp(-a\Delta_{ik})},
\end{equation}
then
\begin{equation}
\frac{\partial\ell_i^{\mathrm{NCE}}}{\partial\Delta_{ii}}
=a(1-P_{ii}),
\qquad
\frac{\partial\ell_i^{\mathrm{NCE}}}{\partial\Delta_{ij}}
=-aP_{ij}\quad(j\ne i).
\end{equation}
InfoNCE therefore also attracts the positive pair.  Pairwise sigmoid calibrates distance on an absolute scale, whereas row-wise InfoNCE normalizes it relative to the other candidates in the query row.  Both retain the radial dependence induced by the distance score.

\subsection{The bounded RAOC score used for training}
\label{app:raoc-score-note}

The geometric relation is Eq.~\eqref{eq:raoc-region}; the bounded score in Eqs.~\eqref{eq:raoc-margins}--\eqref{eq:raoc-score} provides its optimization interface.  Its sign matches geometric membership exactly:
\begin{equation}
D_{\mathrm{RAOC}}(\mathbf{p},\mathbf{q})\ge0
\quad\Longleftrightarrow\quad
\mathbf{q}\in\mathcal C_\beta(\mathbf{p}).
\label{eq:app-score-sign}
\end{equation}
Indeed, $m_r\ge0$ is equivalent to $g\ge0$, and on that domain $m_a\ge0$ is equivalent to $\theta_{pq}\le\min\{\pi,\beta g\}$.  Clipping bounds the optimization score and preserves the zero-threshold RAOC membership.  The contrastive entailment loss uses this continuous score, and the order analysis above uses the corresponding exact relation.

For negative gaps, the angular allowance is clipped to zero,
while $m_r<0$ ensures a negative entailment score even when
$\theta_{pq}=0$.

\section{Relation to Recent Hyperbolic Vision--Language Objectives}
\label{app:recent-hvlm}

\paragraph{Exterior-angle alignment.}
ATMG~\citep{ramasinghe2024accept} analyzes the tension between geodesic-proximity alignment and hierarchical organization and introduces an exterior-angle contrastive objective.  Its score reduces the dependence of semantic matching on absolute radial position.  AURORA uses the common-origin root angle, whose Lorentz and Poincar\'e realizations are given in Eq.~\eqref{eq:app-root-poincare-equivalence}; RAOC then carries the radial structure required for hierarchical ordering.

\paragraph{Parent-dependent entailment cones.}
MERU and HyCoCLIP employ EC descendant regions whose aperture varies with the parent norm, while PHyCLIP applies hierarchy constraints in a factorized hyperbolic representation~\citep{desai2023meru,pal2025hycoclip,yoshikawa2026phyclip}.  Section~\ref{app:alignment-analysis} shows how the EC aperture broadens as the parent approaches its admissible inner radius.  ARGENT~\citep{huynh2026argent} addresses this training regime with adaptive exterior-angle supervision and radial stabilization, and a later operating-point audit reports saturation or weak activation of conventional cone/traversal mechanisms in several released hyperbolic VLM checkpoints~\citep{kim2026operating}.  RAOC parameterizes descendant selectivity directly through the relative radial gap and root angle; the resulting relation admits the transitivity and descendant-set nesting guarantees in Theorem~\ref{thm:partial-order}.

\paragraph{Uncertainty-guided compositional alignment.}
UNCHA is a concurrent 2026 method that models part-to-whole semantic representativeness through a norm-derived uncertainty signal~\citep{kim2026uncha}.  The uncertainty weights its Lorentz-distance contrastive alignment and calibrates an entropy-regularized entailment term.  Its entailment branch is formulated around observed part--whole relations, whereas AURORA assigns entailment labels to both matched and unmatched parent--candidate pairs.  The two methods therefore organize compositional supervision differently: UNCHA modulates distance-based alignment according to part representativeness, and AURORA separates semantic matching into root-angle alignment while learning radial order through contrastive RAOC supervision.

During our reproduction, we found that the public \texttt{uncha/data/webdataset\_mapper.py} contains the same evaluation-dataset implementation as \texttt{uncha/data/evaluation.py}, leaving the training WebDataset mapper implementation absent from the released code path~\citep{uncha2026code}.  We reconstructed the mapper from the data pipeline described in the paper and the surrounding training interfaces.  Since the resulting implementation did not consistently reproduce the reported trends across repeated runs, we exclude reproduced UNCHA numbers from the controlled comparison tables to avoid comparisons based on an implementation whose fidelity could not be fully verified.

\section{Implementation Details}
\label{app:implementation}

\paragraph{Model architecture.}
The text branch uses the CLIP Transformer configuration with 12 layers, hidden width 512, maximum sequence length 77, and vocabulary size 49,408~\citep{radford2021clip}.  The visual branch uses ViT-S/16 for the principal experiments and ViT-B/16 for the larger-backbone comparison, with $16\times16$ image patches~\citep{dosovitskiy2021vit}.  Images are resized with border padding and random-resized cropping to $224\times224$; the crop scale is sampled from $[0.5,1.0]$.  This yields a $14\times14$ patch grid (196 image tokens), to which fixed two-dimensional sine--cosine positional embeddings are added.  Full images and object crops share the visual encoder, while captions and object phrases share the text encoder.

\paragraph{Training data and hyperbolic parameterization.}
We use the recoverable subset of GRIT~\citep{peng2023kosmos2}, containing 10.3M valid image--caption pairs and 17.8M grounded box--phrase annotations after local recovery and validation.  All four representation types $(I,T,I^b,T^b)$ are projected to the same $d$-dimensional Lorentz space.  The shared curvature magnitude $c$ is learned globally, initialized at $1.0$, and constrained to $[0.1,10]$.  The principal ViT-S/16 configuration uses
\begin{equation}
\lambda_e=0.15,
\qquad
\lambda_n=0.05,
\qquad
\beta=8,
\qquad
\sigma_r=0.1.
\end{equation}
The parameter $\beta$ determines the geometric expansion rate of RAOC and is therefore part of the order definition.  The parameter $\sigma_r$ sets the scale of the bounded radial optimization margin while preserving zero-threshold membership.  We use the norm regularizer defined in Sec.~\ref{sec:contrastive-entailment} to penalize both very small and excessively large spatial norms, stabilizing training and maintaining a usable radial range for hierarchical organization.

\paragraph{Optimizer, schedule, and hardware.}
All principal runs use four NVIDIA A40 48GB GPUs, a global batch size of 768, and 500k optimization steps.  We use AdamW~\citep{loshchilov2019adamw} with $\beta_1=0.9$, $\beta_2=0.98$, and weight decay $0.2$; weight decay is disabled for learned scalar parameters.  The learning rate is linearly warmed up for the first 4k steps to a maximum of $5\times10^{-4}$ and then follows cosine decay~\citep{loshchilov2017sgdr}.  We keep these settings fixed across all principal AURORA runs and downstream benchmarks.

\paragraph{Training efficiency.}
We profile ViT-S/16 training under the common setup of four NVIDIA A40 48GB GPUs, global batch size 768, and 500k optimization steps. HyCoCLIP and PHyCLIP are reproduced using their original source code without implementation changes. Our implementation streamlines the training computation branches and eliminates redundant matrix computations, while keeping all other training settings the same. The efficiency comparison therefore reflects these implementation optimizations as well as differences in the training objectives.

Table~\ref{tab:compute} reports the mean peak memory and end-to-end training time over three runs with randomly selected seeds. AURORA uses 13.8GB per GPU and completes training in 77.51 hours, compared with 39.4GB and 91.8 hours for HyCoCLIP and 40.9GB and 93.4 hours for PHyCLIP. This corresponds to about 65--66\% lower peak VRAM and 16--17\% shorter training time under the matched environment. The AURORA-InfoNCE variant remains below both hyperbolic baselines at 21.2GB per GPU and 80.83 hours. Together with the stronger downstream results in Sec.~\ref{sec:experimental-results}, these measurements give AURORA a favorable accuracy--efficiency profile under the shared hardware and schedule.

\begin{table}[H]
\centering
\caption{ViT-S/16 training efficiency in the common four-A40 environment. Values are means over three runs with randomly selected seeds.}
\label{tab:compute}
\small
\resizebox{0.7\linewidth}{!}{%
\begin{tabular}{lccc}
\toprule
Model & GPUs & Peak VRAM/GPU (GB) & Train time (h) \\
\midrule
HyCoCLIP & $4\times$A40 & 39.4 & 91.80 \\
PHyCLIP & $4\times$A40 & 40.9 & 93.40 \\
AURORA-InfoNCE & $4\times$A40 & 21.2 & 80.83 \\
\rowcolor{oursRow}AURORA & $4\times$A40 & \textbf{13.8} & \textbf{77.51} \\
\bottomrule
\end{tabular}}
\vspace{-5pt}
\end{table}

\section{Complete Experimental Tables}
\label{app:fulltables}

We report complete results for ViT-S/16 and ViT-B/16, extending the comparisons in the main text to both model sizes.

\paragraph{Zero-shot image classification.}
Table~\ref{tab:classification-full} shows that AURORA achieves the highest accuracy on 13 of the 16 datasets with both backbones. Its advantages span general, fine-grained, and specialized recognition, with particularly clear gains on ImageNet, Caltech-101, Pets, and Flowers. Moving from ViT-S/16 to ViT-B/16 improves AURORA on every dataset while preserving its overall ranking relative to the baselines. The results therefore extend the classification gains observed in the main text to the larger backbone.

\begin{table}[H]
\centering
\caption{\textbf{Complete zero-shot image classification comparison for ViT-S/16 and ViT-B/16.} Top-1 accuracy (\%). ViT-S/16 and ViT-B/16 best results are shown in \textbf{bold} and second-best results are \underline{underlined}.}
\label{tab:classification-full}
\scriptsize
\setlength{\tabcolsep}{1.85pt}
\renewcommand{\arraystretch}{1.04}
\resizebox{0.985\textwidth}{!}{%
\begin{tabular}{@{}c l c *{16}{r}@{}}
\toprule
\multicolumn{3}{c}{} &
\multicolumn{6}{c}{\cellcolor{generalHead}\textbf{General datasets}} &
\multicolumn{6}{c}{\cellcolor{fineHead}\textbf{Fine-grained datasets}} &
\multicolumn{4}{c}{\cellcolor{specialHead}\textbf{Specialized datasets}}\\
\cmidrule(lr){4-9}\cmidrule(lr){10-15}\cmidrule(lr){16-19}
\textbf{Backbone} & \textbf{Model} & \rot{w/ boxes} &
\cellcolor{generalCell}\rot{ImageNet} &
\cellcolor{generalCell}\rot{CIFAR-10} &
\cellcolor{generalCell}\rot{CIFAR-100} &
\cellcolor{generalCell}\rot{SUN397} &
\cellcolor{generalCell}\rot{Caltech-101} &
\cellcolor{generalCell}\rot{STL-10} &
\cellcolor{fineCell}\rot{Food-101} &
\cellcolor{fineCell}\rot{CUB} &
\cellcolor{fineCell}\rot{Cars} &
\cellcolor{fineCell}\rot{Aircraft} &
\cellcolor{fineCell}\rot{Pets} &
\cellcolor{fineCell}\rot{Flowers} &
\cellcolor{specialCell}\rot{DTD} &
\cellcolor{specialCell}\rot{EuroSAT} &
\cellcolor{specialCell}\rot{RESISC45} &
\cellcolor{specialCell}\rot{Country211}\\
\midrule

\multirow{11}{*}{\makecell{ViT\\S/16}}
& CLIP & \xmark & 32.66&74.19&42.22&42.91&69.80&90.09&42.72&12.55&3.93&2.06&43.22&18.03&17.57&25.47&33.64&3.88\\
& CLIP & \cmark & 32.49&73.69&42.15&42.66&70.07&89.83&42.50&12.48&4.01&2.10&43.28&17.79&17.61&26.04&33.37&3.84\\
& SigLIP & \xmark & 34.30&66.57&38.38&44.73&71.42&91.75&38.98&13.33&4.89&2.13&46.65&\underline{23.17}&18.40&\textbf{32.63}&35.42&4.15\\
& SigLIP & \cmark & 34.34&67.07&38.50&44.82&71.55&91.82&39.27&13.33&4.92&2.24&46.56&23.03&18.36&\underline{31.42}&35.37&\underline{4.23}\\
& MERU & \xmark & 31.74&76.08&41.24&41.28&67.76&87.24&41.09&11.97&3.80&\underline{2.88}&42.53&16.75&17.46&27.75&33.88&3.82\\
& MERU & \cmark & 32.01&76.55&41.35&41.21&67.64&87.30&41.24&12.08&3.78&\textbf{2.95}&42.54&16.57&17.52&27.53&33.82&3.83\\
& ATMG & \xmark & 33.74&77.21&43.28&43.15&68.08&90.84&44.01&13.27&4.38&1.75&42.53&15.70&17.89&28.69&35.75&4.21\\
& HyCoCLIP & \cmark & 36.38&83.28&47.91&47.88&71.53&93.16&46.73&15.33&5.98&2.45&45.59&20.89&19.47&26.00&38.73&4.09\\
& PHyCLIP & \cmark & 37.06&83.44&48.41&48.88&71.82&93.06&47.26&16.29&6.58&2.57&46.44&21.46&19.13&25.81&39.79&3.95\\
& ARGENT & \cmark & \underline{37.55}&\underline{83.55}&\underline{48.62}&\underline{49.29}&\underline{72.15}&\textbf{94.01}&\underline{47.70}&\underline{16.90}&\underline{6.81}&2.48&\underline{47.10}&21.65&\underline{19.67}&26.32&\underline{40.46}&4.19\\
\rowcolor{oursRow}
& \textbf{AURORA} & \cmark & \textbf{41.54}&\textbf{84.25}&\textbf{51.20}&\textbf{50.27}&\textbf{75.81}&\underline{93.19}&\textbf{50.13}&\textbf{17.71}&\textbf{8.91}&1.80&\textbf{51.41}&\textbf{28.99}&\textbf{22.87}&30.18&\textbf{41.09}&\textbf{4.70}\\

\midrule

\multirow{8}{*}{\makecell{ViT\\B/16}}

& CLIP & \xmark 
&36.28&78.42&48.72&46.88&72.54&92.51
&47.18&13.92&5.72&4.12&46.82&22.31
&18.62&28.73&36.18&5.02\\

& SigLIP & \xmark 
&38.62&74.01&47.34&49.82&75.11&94.18
&43.82&15.28&7.12&4.30&51.62&\underline{29.18}
&19.96&\textbf{36.91}&38.41&5.43\\

& MERU & \xmark 
&35.41&79.62&47.91&45.54&70.38&90.21
&45.46&13.61&5.38&\textbf{5.51}&46.34&20.74
&18.41&31.02&36.37&4.86\\

& ATMG & \xmark 
&37.83&80.92&50.22&47.56&71.32&93.36
&48.44&15.14&6.21&3.58&46.77&19.42
&19.02&32.48&38.32&5.34\\

& HyCoCLIP & \cmark 
&40.51&87.61&55.71&52.63&74.82&95.47
&51.82&17.43&8.42&4.78&50.11&25.04
&20.62&29.66&41.61&5.28\\

& PHyCLIP & \cmark 
&41.63&87.86&56.44&53.91&75.19&95.81
&52.74&18.21&9.31&\underline{5.18}&50.94&26.41
&20.39&29.21&42.72&5.19\\

& ARGENT & \cmark 
&\underline{42.31}&\underline{88.02}&\underline{57.03}&\underline{54.37}
&\underline{75.62}&\textbf{96.46}
&\underline{53.42}&\underline{18.96}&\underline{9.88}&4.96
&\underline{51.83}&27.46
&\underline{21.11}&29.83&\underline{43.67}&\underline{5.57}\\

\rowcolor{oursRow}
& \textbf{AURORA} & \cmark 
&\textbf{46.43}&\textbf{88.57}&\textbf{59.83}&\textbf{55.51}
&\textbf{79.12}&\underline{95.96}
&\textbf{55.71}&\textbf{19.94}&\textbf{12.70}&3.54
&\textbf{56.42}&\textbf{34.46}
&\textbf{24.47}&\underline{33.86}&\textbf{44.22}&\textbf{6.09}\\
\bottomrule
\end{tabular}}
\end{table}

\paragraph{Image--text retrieval and hierarchical classification.}
As shown in Table~\ref{tab:retrieval-hierarchy-full}, AURORA leads all reported retrieval and hierarchical classification metrics with both backbones. The retrieval advantage holds in both directions on COCO and Flickr30K, at Recall@5 and Recall@10. With ViT-B/16, AURORA further improves retrieval performance and reduces hierarchical errors, accompanied by higher Jaccard similarity, hierarchical precision, and hierarchical recall. Thus, stronger semantic matching and closer agreement with the WordNet hierarchy are observed at both model sizes.

\begin{table}[H]
\centering
\caption{\textbf{Complete zero-shot retrieval and hierarchical classification comparison for ViT-S/16 and ViT-B/16.} Retrieval results are Recall@5/10 (\%). Hierarchical metrics are rounded to two decimal places for display; best and second-best markings follow the underlying values.}
\label{tab:retrieval-hierarchy-full}
\scriptsize
\setlength{\tabcolsep}{2.55pt}
\renewcommand{\arraystretch}{1.05}
\resizebox{\textwidth}{!}{%
\begin{tabular}{@{}c l c *{13}{c}@{}}
\toprule
\multicolumn{3}{c}{} &
\multicolumn{4}{c}{\cellcolor{generalHead}\textbf{Text $\rightarrow$ Image}} &
\multicolumn{4}{c}{\cellcolor{fineHead}\textbf{Image $\rightarrow$ Text}} &
\multicolumn{5}{c}{\cellcolor{specialHead}\textbf{Hierarchical Classification}}\\
\cmidrule(lr){4-7}\cmidrule(lr){8-11}\cmidrule(lr){12-16}
\textbf{Backbone} & \textbf{Model} & \rot{w/ boxes} &
\multicolumn{2}{c}{\cellcolor{generalCell}\textbf{COCO}} &
\multicolumn{2}{c}{\cellcolor{generalCell}\textbf{Flickr}} &
\multicolumn{2}{c}{\cellcolor{fineCell}\textbf{COCO}} &
\multicolumn{2}{c}{\cellcolor{fineCell}\textbf{Flickr}} &
\multicolumn{5}{c}{\cellcolor{specialCell}\textbf{WordNet}}\\
& & & R@5 & R@10 & R@5 & R@10 & R@5 & R@10 & R@5 & R@10 &
TIE$\downarrow$ & LCA$\downarrow$ & $J\uparrow$ & $P_H\uparrow$ & $R_H\uparrow$\\
\midrule

\multirow{11}{*}{\makecell{ViT\\S/16}}
& CLIP & \xmark
& 50.31&61.52&76.98&84.48&65.75&76.15&87.10&93.00&4.43&2.51&0.72&0.79&0.80\\

& CLIP & \cmark
& 50.27&61.47&76.75&84.42&65.49&76.08&87.20&93.17&4.42&2.50&0.72&0.78&0.80\\

& SigLIP & \xmark
& 52.32&63.90&77.54&85.24&65.98&76.78&\underline{88.60}&93.23&4.22&2.47&0.75&0.83&0.82\\

& SigLIP & \cmark
& 52.27&63.81&77.21&84.09&\underline{66.03}&\underline{76.84}&\underline{88.60}&\underline{93.40}&4.23&2.47&0.74&0.83&0.83\\

& MERU & \xmark
& 49.79&61.14&76.03&84.34&65.01&75.67&86.80&92.77&4.44&2.52&0.71&0.78&0.79\\

& MERU & \cmark
& 49.63&61.08&76.13&84.21&65.16&75.74&86.93&92.77&4.44&2.51&0.71&0.78&0.78\\

& ATMG & \xmark
& 50.95&62.36&77.21&84.74&64.96&75.81&87.10&93.17&4.37&2.48&0.73&0.82&0.81\\

& HyCoCLIP & \cmark
& 51.35&63.13&77.92&85.52&64.84&75.78&87.30&93.23&3.97&2.33&0.76&0.84&0.84\\

& PHyCLIP & \cmark
& 52.24&64.08&78.26&85.84&65.65&76.21&87.20&93.17&3.92&2.31&0.76&\underline{0.84}&\underline{0.84}\\

& ARGENT & \cmark
& \underline{52.68}&\underline{64.46}&\underline{78.59}&\underline{86.18}&65.75&76.46&87.50&93.23&
\underline{3.91}&\underline{2.30}&\underline{0.77}&\underline{0.84}&\underline{0.84}\\

\rowcolor{oursRow}
& \textbf{AURORA} & \cmark
& \textbf{55.77}&\textbf{66.81}&\textbf{81.02}&\textbf{88.18}
& \textbf{69.42}&\textbf{79.18}&\textbf{89.90}&\textbf{95.13}
& \textbf{3.62}&\textbf{2.25}&\textbf{0.78}&\textbf{0.85}&\textbf{0.85}\\

\midrule

\multirow{8}{*}{\makecell{ViT\\B/16}}
& CLIP & \xmark
& 53.08&64.36&79.73&86.63&68.32&78.56&89.63&94.16
& 3.88&2.34&0.77&0.83&0.84\\

& SigLIP & \xmark
& 54.62&65.92&80.42&87.31&69.15&79.36&\underline{92.03}&\underline{95.63}
& 3.69&2.29&0.78&0.84&0.85\\

& MERU & \xmark
& 52.71&63.98&79.10&86.18&67.69&77.83&88.97&93.91
& 3.97&2.36&0.76&0.83&0.83\\

& ATMG & \xmark
& 53.94&65.14&80.09&86.82&68.47&78.54&89.88&94.37
& 3.80&2.31&0.77&0.84&0.84\\

& HyCoCLIP & \cmark
& 55.12&66.41&81.30&87.78&69.84&79.72&90.63&94.93
& 3.48&2.16&0.79&0.85&0.85\\

& PHyCLIP & \cmark
& 55.83&67.07&82.03&88.31&70.72&80.61&90.72&95.12
& 3.40&2.14&\underline{0.80}&\underline{0.86}&\underline{0.86}\\

& ARGENT & \cmark
& \underline{56.52}&\underline{67.86}&\underline{82.74}&\underline{88.87}
& \underline{71.36}&\underline{81.23}&91.48&95.40
& \underline{3.34}&\underline{2.12}&\underline{0.80}&\underline{0.86}&\underline{0.86}\\

\rowcolor{oursRow}
& \textbf{AURORA} & \cmark
& \textbf{59.41}&\textbf{70.10}&\textbf{85.60}&\textbf{90.70}
& \textbf{73.82}&\textbf{83.18}&\textbf{92.50}&\textbf{96.30}
& \textbf{3.18}&\textbf{2.07}&\textbf{0.81}&\textbf{0.87}&\textbf{0.87}\\

\bottomrule
\end{tabular}}
\end{table}

\paragraph{Compositional understanding.}
Table~\ref{tab:composition-full} shows a consistent pattern across both backbones: AURORA leads on nine of the fourteen measures, including the overall SugarCrepe score. On VL-CheckList-Object, it performs best for objects in central and intermediate locations and for medium and small objects. On SugarCrepe, its strengths lie in replacing and adding objects or attributes, while relation replacement and swapping remain more challenging. These trends persist with ViT-B/16, indicating that the compositional benefits extend to the larger model.

\begin{table}[H]
\centering
\caption{\textbf{Complete hard-negative compositional evaluation for ViT-S/16 and ViT-B/16.}}
\label{tab:composition-full}
\scriptsize
\setlength{\tabcolsep}{2.45pt}
\renewcommand{\arraystretch}{1.05}
\resizebox{\textwidth}{!}{%
\begin{tabular}{@{}c l c *{14}{c}@{}}
\toprule
\multicolumn{3}{c}{} &
\multicolumn{6}{c}{\cellcolor{generalHead}\textbf{VL-CheckList-Object}} &
\multicolumn{8}{c}{\cellcolor{fineHead}\textbf{SugarCrepe}}\\
\cmidrule(lr){4-9}\cmidrule(lr){10-17}
\textbf{Backbone} & \textbf{Model} & \rot{w/ boxes} &
\multicolumn{3}{c}{\cellcolor{generalCell}\textbf{Location}} &
\multicolumn{3}{c}{\cellcolor{generalCell}\textbf{Size}} &
\multicolumn{3}{c}{\cellcolor{fineCell}\textbf{Replace}} &
\multicolumn{2}{c}{\cellcolor{fineCell}\textbf{Swap}} &
\multicolumn{3}{c}{\cellcolor{fineCell}\textbf{Add}}\\
& & & Center & Mid & Margin & Large & Medium & Small &
Obj & Att & Rel & Obj & Att & Obj & Att & Overall\\
\midrule

\multirow{11}{*}{\makecell{ViT\\S/16}}
& CLIP & \xmark
& 85.10&79.35&71.22&84.48&74.61&65.72&
86.74&78.00&69.04&59.49&65.27&77.53&70.76&75.71\\

& CLIP & \cmark
& 85.24&78.41&70.55&85.32&75.02&64.67&
85.94&77.45&\textbf{69.94}&59.89&64.41&77.90&70.76&75.69\\

& SigLIP & \xmark
& 89.33&83.04&74.91&88.30&78.50&68.82&
87.35&78.55&\underline{69.49}&\underline{60.82}&\underline{65.77}&78.08&71.10&76.26\\

& SigLIP & \cmark
& 89.83&83.03&74.50&88.13&78.13&69.47&
87.01&78.43&69.18&\textbf{61.38}&\textbf{66.67}&78.03&70.28&76.12\\

& MERU & \xmark
& 82.07&76.10&67.93&81.47&71.50&62.40&
86.40&77.58&68.73&58.94&64.96&77.17&70.33&75.35\\

& MERU & \cmark
& 82.90&75.40&68.40&81.33&71.73&62.40&
87.09&77.03&69.11&58.81&65.17&77.19&70.47&75.55\\

& ATMG & \xmark
& 84.07&78.40&69.83&83.53&73.60&64.53&
87.11&78.34&69.35&60.16&65.62&77.98&71.15&76.10\\

& HyCoCLIP & \cmark
& 90.80&86.75&79.95&90.52&82.73&74.23&
87.23&78.81&66.64&60.41&64.86&79.44&69.08&75.82\\

& PHyCLIP & \cmark
& 91.40&87.14&80.42&90.97&83.06&74.79&
88.40&80.25&67.33&59.76&65.22&79.89&71.48&76.71\\

& ARGENT & \cmark
& \underline{91.87}&\underline{87.73}&\textbf{81.10}&\textbf{91.33}&\underline{84.07}&\underline{75.60}&
\underline{88.92}&\underline{80.88}&67.71&59.89&65.32&\underline{80.00}&\underline{72.59}&\underline{77.11}\\

\rowcolor{oursRow}
& \textbf{AURORA} & \cmark
& \textbf{91.93}&\textbf{87.80}&\underline{80.77}&\underline{91.30}&\textbf{84.69}&\textbf{76.53}&
\textbf{89.77}&\textbf{81.85}&68.42&60.00&65.47&\textbf{80.16}&\textbf{74.27}&\textbf{77.75}\\

\midrule

\multirow{8}{*}{\makecell{ViT\\B/16}}
& CLIP & \xmark
& 86.18&80.72&72.64&85.52&76.15&67.43&
88.00&78.72&70.11&59.81&65.54&80.63&72.52&77.31\\

& SigLIP & \xmark
& 90.30&84.35&76.24&89.13&79.79&70.32&
88.75&79.35&\textbf{70.68}&\textbf{61.11}&\textbf{66.03}&81.44&73.12&78.02\\

& MERU & \xmark
& 84.28&78.29&70.22&83.27&74.02&65.19&
87.55&78.20&69.78&59.28&65.24&80.16&72.10&76.89\\

& ATMG & \xmark
& 86.01&80.70&72.45&85.37&76.10&67.10&
88.41&79.03&\underline{70.50}&60.47&\underline{65.91}&81.16&73.07&77.76\\

& HyCoCLIP & \cmark
& 91.88&88.12&81.47&91.54&84.29&75.95&
88.52&79.50&67.81&\underline{60.72}&65.15&82.54&71.00&77.46\\

& PHyCLIP & \cmark
& 92.45&88.53&82.02&91.93&84.62&76.56&
89.67&80.93&68.50&60.07&65.50&83.04&73.43&78.36\\

& ARGENT & \cmark
& \underline{92.78}&\underline{89.17}&\textbf{82.63}&\textbf{92.45}&\underline{85.55}&\underline{77.22}&
\underline{90.26}&\underline{81.57}&68.90&60.21&65.62&\underline{83.38}&\underline{74.54}&\underline{78.85}\\

\rowcolor{oursRow}
& \textbf{AURORA} & \cmark
& \textbf{93.01}&\textbf{89.30}&\underline{82.42}&\underline{92.31}&\textbf{86.25}&\textbf{78.38}&
\textbf{91.10}&\textbf{82.49}&69.84&60.41&65.77&\textbf{83.95}&\textbf{76.59}&\textbf{79.67}\\

\bottomrule
\end{tabular}}
\end{table}

\section{Evaluation Protocols}
\label{app:protocols}

\subsection{Zero-shot image classification}

We evaluate the same 16 classification datasets for every checkpoint: ImageNet~\citep{deng2009imagenet}, CIFAR-10~\citep{krizhevsky2009cifar}, CIFAR-100~\citep{krizhevsky2009cifar}, SUN397~\citep{xiao2010sun}, Caltech-101~\citep{feifei2004caltech}, STL-10~\citep{coates2011stl}, Food-101~\citep{bossard2014food}, CUB-200-2011~\citep{wah2011cub}, Stanford Cars~\citep{krause2013cars}, FGVC Aircraft~\citep{maji2013aircraft}, Oxford-IIIT Pets~\citep{parkhi2012pets}, Oxford Flowers-102~\citep{nilsback2008flowers}, DTD~\citep{cimpoi2014dtd}, EuroSAT~\citep{helber2019eurosat}, RESISC45~\citep{cheng2017resisc45}, and Country211~\citep{radford2021clip}.  For a given dataset, all models use the same class vocabulary and prompt templates.  AURORA scores image embeddings against text prototypes using $s_\angle$; each baseline is evaluated with the similarity function native to that model.  We report top-1 accuracy.

\subsection{Image--text retrieval}

Bidirectional retrieval is evaluated on the COCO validation set~\citep{lin2014coco} and the Flickr30K test set~\citep{young2014flickr}.  In text-to-image retrieval, each caption queries the image collection; in image-to-text retrieval, each image queries the caption collection.  We report Recall@5 and Recall@10 in both directions, matching the columns in the main and appendix tables.  AURORA ranks candidates by root-angle similarity, while each baseline retains its native released retrieval score.

\subsection{Hierarchical classification on WordNet}

We evaluate whether zero-shot ImageNet predictions remain close to the ground-truth category in the WordNet hierarchy~\citep{deng2009imagenet,miller1995wordnet}.  Let $\hat y$ and $y$ denote the predicted and true synsets.  Tree-Induced Error (TIE) measures the shortest hierarchical path cost between them~\citep{dekel2004hierarchical}:
\begin{equation}
\operatorname{TIE}(\hat y,y)
=
\sum_{e\in P(\hat y,y)}w_e,
\end{equation}
where all WordNet edge weights are set to $w_e=1$.  Let $a^\star$ be the deepest common ancestor of $\hat y$ and $y$.  Following the hierarchical-classification protocol used by the compared hyperbolic VLMs~\citep{pal2025hycoclip,yoshikawa2026phyclip}, the LCA error is the larger of the two distances from the labels to this common ancestor,
\begin{equation}
\operatorname{LCA}(\hat y,y)
=
\max\!\left\{
 d_{\mathcal G}(\hat y,a^\star),
 d_{\mathcal G}(y,a^\star)
\right\}.
\end{equation}
Lower TIE and LCA indicate hierarchically closer errors.

For set-based measures, let $\widehat Y_{\mathrm{anc}}$ and $Y_{\mathrm{anc}}$ be the ancestor sets of $\hat y$ and $y$.  We report Jaccard similarity, hierarchical precision, and hierarchical recall~\citep{kosmopoulos2015evaluation}:
\begin{equation}
J
=
\frac{|\widehat Y_{\mathrm{anc}}\cap Y_{\mathrm{anc}}|}
{|\widehat Y_{\mathrm{anc}}\cup Y_{\mathrm{anc}}|},
\qquad
P_H
=
\frac{|\widehat Y_{\mathrm{anc}}\cap Y_{\mathrm{anc}}|}
{|\widehat Y_{\mathrm{anc}}|},
\qquad
R_H
=
\frac{|\widehat Y_{\mathrm{anc}}\cap Y_{\mathrm{anc}}|}
{|Y_{\mathrm{anc}}|}.
\label{eq:app-hierarchical-metrics}
\end{equation}
Higher $J$, $P_H$, and $R_H$ indicate better agreement with the reference taxonomy.

\subsection{VL-CheckList--Object}

VL-CheckList~\citep{zhao2022vlchecklist} constructs controlled image--text tests by modifying a caption while keeping the image fixed.  We use the object subset (VL-CO), in which an object noun phrase is replaced to create an incorrect alternative.  Results are grouped by the scale and image location of the referenced object.

For object scale, let $A$ denote the bounding-box area.  The three groups are
\begin{equation}
\text{O-Small}: A<32^2,
\qquad
\text{O-Medium}: 32^2\le A\le96^2,
\qquad
\text{O-Large}: A>96^2.
\end{equation}
For location, let $x$ be the image half-diagonal and $y$ the Euclidean distance between the object-box center and the image center.  The groups are
\begin{equation}
\text{O-Center}: \frac{y}{x}\le\frac13,
\qquad
\text{O-Mid}: \frac13<\frac{y}{x}\le\frac23,
\qquad
\text{O-Margin}: \frac{y}{x}>\frac23.
\end{equation}
For each subset, accuracy is the fraction of cases in which the model assigns the correct image--text pair a higher compatibility score than the corrupted alternative.

\subsection{SugarCrepe}

SugarCrepe~\citep{hsieh2023sugarcrepe} evaluates image-to-text compositional discrimination using fluent hard-negative captions.  We report the seven released perturbation categories: \textsc{Replace-Object}, \textsc{Replace-Attribute}, \textsc{Replace-Relation}, \textsc{Swap-Object}, \textsc{Swap-Attribute}, \textsc{Add-Object}, and \textsc{Add-Attribute}.  Each example asks the model to prefer the original caption over its hard negative; the overall score is the mean over the reported categories.

\section{Additional Hierarchy Visualizations}
\label{app:additional-hierarchy-visualizations}

We provide additional visualizations of AURORA ViT-S with and without RAOC to examine semantic alignment and hierarchical organization together. Each scatter plot shows the root angle $\theta$ to a fixed query and the radial gap $g$ between modalities. Smaller root angles indicate stronger semantic alignment. Dashed lines connect concepts or image regions according to the displayed hierarchy, allowing their radial order to be inspected alongside their similarity to the query. Orange squares denote related alternatives, and green crosses denote unrelated controls. The original axis ranges are retained for each model.

\subsection{Image to Text}
\label{app:hierarchy-i2t}

Figures~\ref{fig:app-i2t-01-05} and~\ref{fig:app-i2t-06-10} compare ten image queries. Each row shows the same example with RAOC on the left and without RAOC on the right. Blue circles trace the semantic path from a broad category to the query class. With RAOC, these paths consistently move toward smaller radial gaps and smaller root angles as the concepts become more specific. Without RAOC, the query class can still have a small root angle, while the radial sequence reverses in examples such as electric guitar, leopard, monarch, and zebra. The piano, laptop, and desk examples retain an ordered path in both variants. Together, these comparisons support the role of RAOC in making hierarchical order more consistent while preserving semantic alignment.

\begin{figure*}[p]
    \centering
    \includegraphics[width=\textwidth,height=0.84\textheight,keepaspectratio]{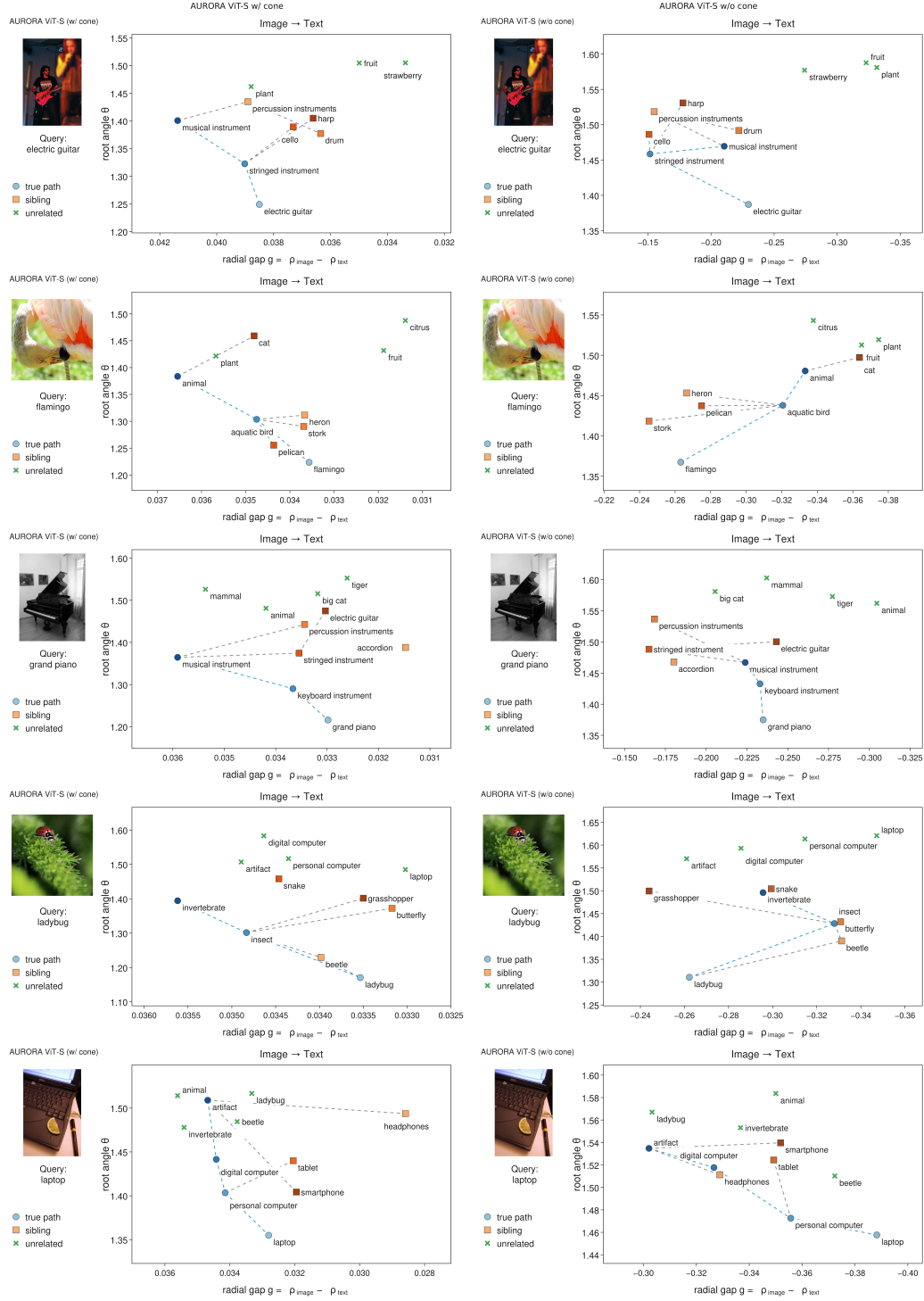}
    \caption{\textbf{Image to text hierarchy comparisons for five object queries.} Rows show electric guitar, flamingo, grand piano, ladybug, and laptop. The left column uses RAOC, and the right column omits it. Dashed blue lines trace the query's semantic path, while gray lines show related branches. RAOC preserves the angular preference for the query class and gives a consistent radial progression along each path. The radial reversals for electric guitar and ladybug without RAOC illustrate why angular alignment alone does not ensure hierarchical order.}
    \label{fig:app-i2t-01-05}
\end{figure*}

\begin{figure*}[p]
    \centering
    \includegraphics[width=\textwidth,height=0.84\textheight,keepaspectratio]{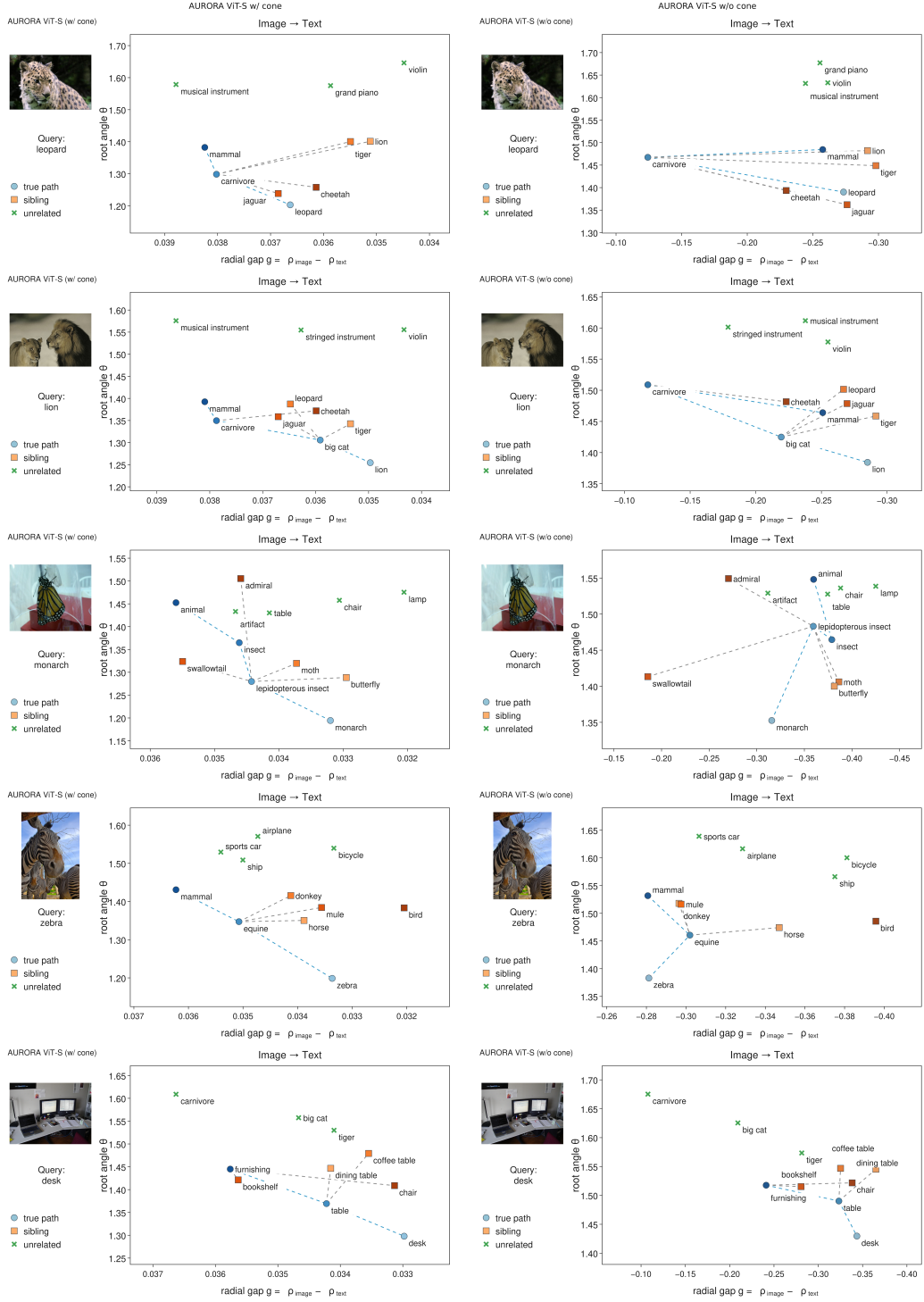}
    \caption{\textbf{Image to text hierarchy comparisons for five further queries.} Rows show leopard, lion, monarch, zebra, and desk, with RAOC on the left and without RAOC on the right. RAOC places the displayed semantic paths in a consistent radial order. Without RAOC, the animal examples exhibit reversals between successive concepts even when the query class remains well aligned. The desk path is ordered in both variants, showing that the effect concerns consistency across examples rather than the absence of all hierarchy without RAOC.}
    \label{fig:app-i2t-06-10}
\end{figure*}

\subsection{Text to Image}
\label{app:hierarchy-t2i}

\begin{figure*}[p]
    \centering
    \includegraphics[width=\textwidth]{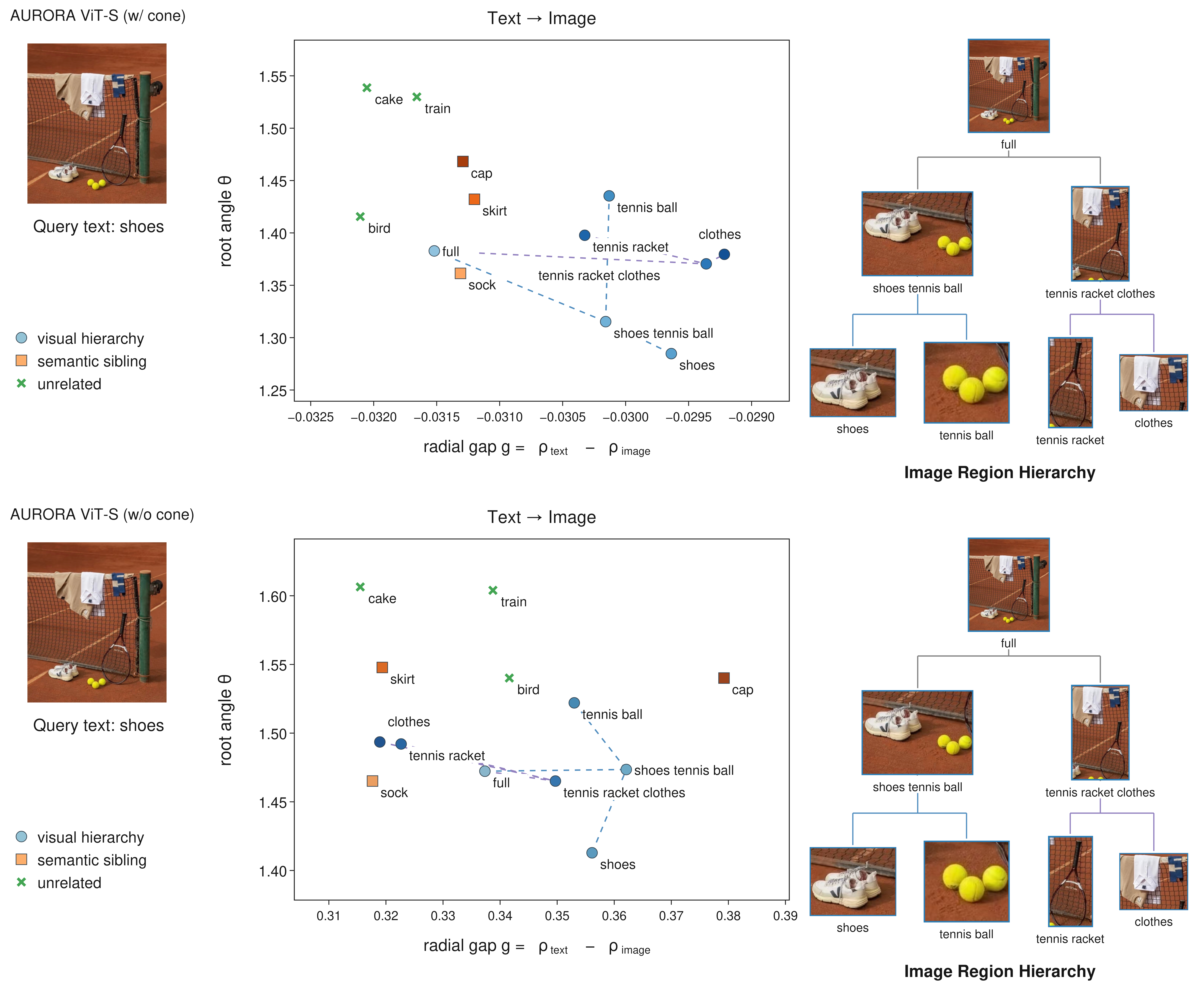}
    \caption{\textbf{Text to image hierarchy for the query ``shoes.''} AURORA ViT-S with RAOC is shown above the variant without RAOC. With RAOC, the full image, shoes and tennis ball region, and shoes crop follow a consistent radial progression as their root angles decrease. Without RAOC, the shoes crop remains the most aligned region, but the final step reverses the radial direction. The region tree provides the common reference hierarchy.}
    \label{fig:app-t2i-shoes}
\end{figure*}

\begin{figure*}[p]
    \centering
    \includegraphics[width=\textwidth]{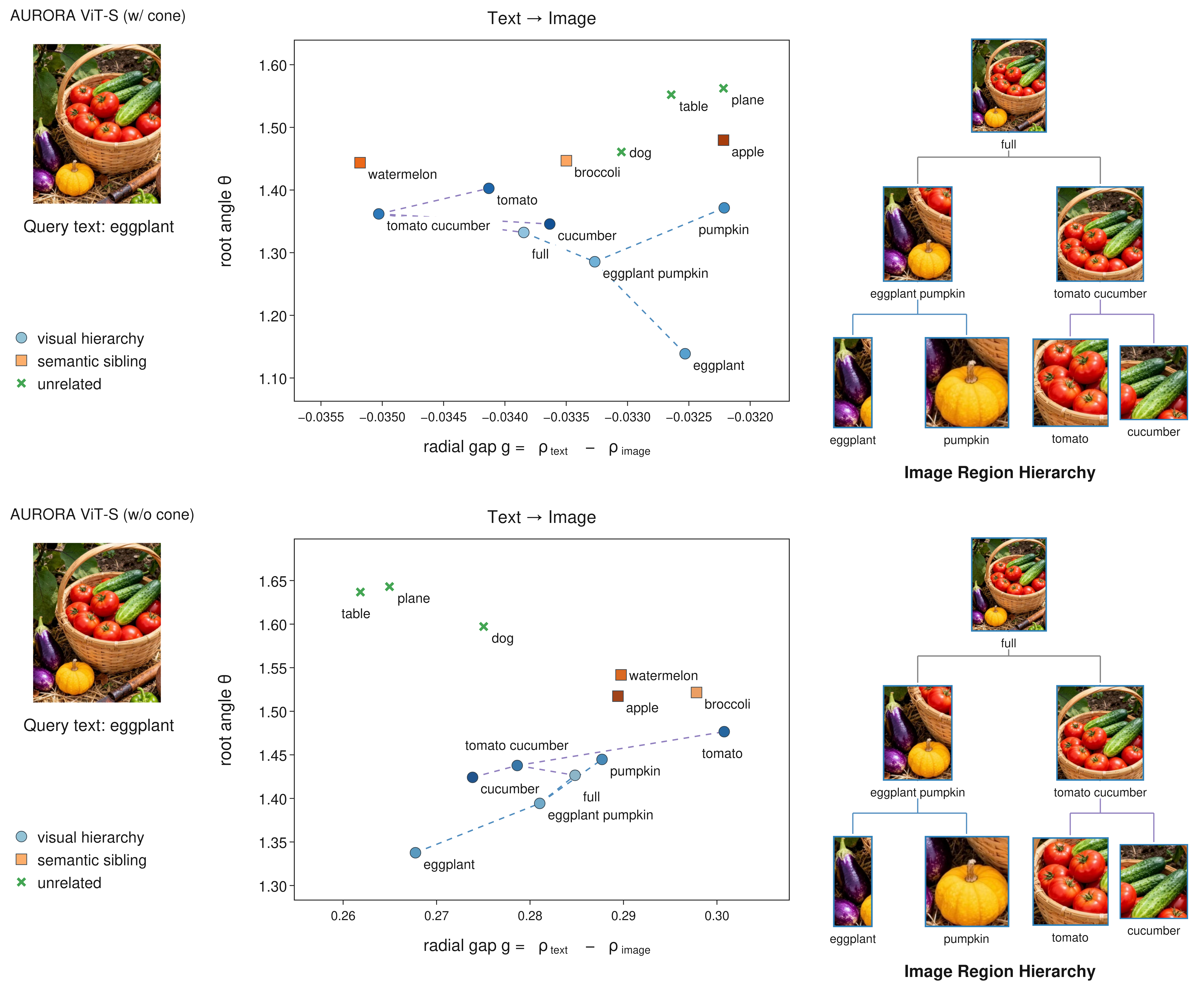}
    \caption{\textbf{Text to image hierarchy for the query ``eggplant.''} The top panel uses RAOC, and the bottom panel omits it. Both models assign the eggplant crop a smaller root angle than the full image. RAOC also places the full image, eggplant and pumpkin region, and eggplant crop in the expected radial order. Without RAOC, this branch progresses in the opposite radial direction despite retaining its angular preference for eggplant.}
    \label{fig:app-t2i-eggplant}
\end{figure*}

\begin{figure*}[p]
    \centering
    \includegraphics[width=\textwidth]{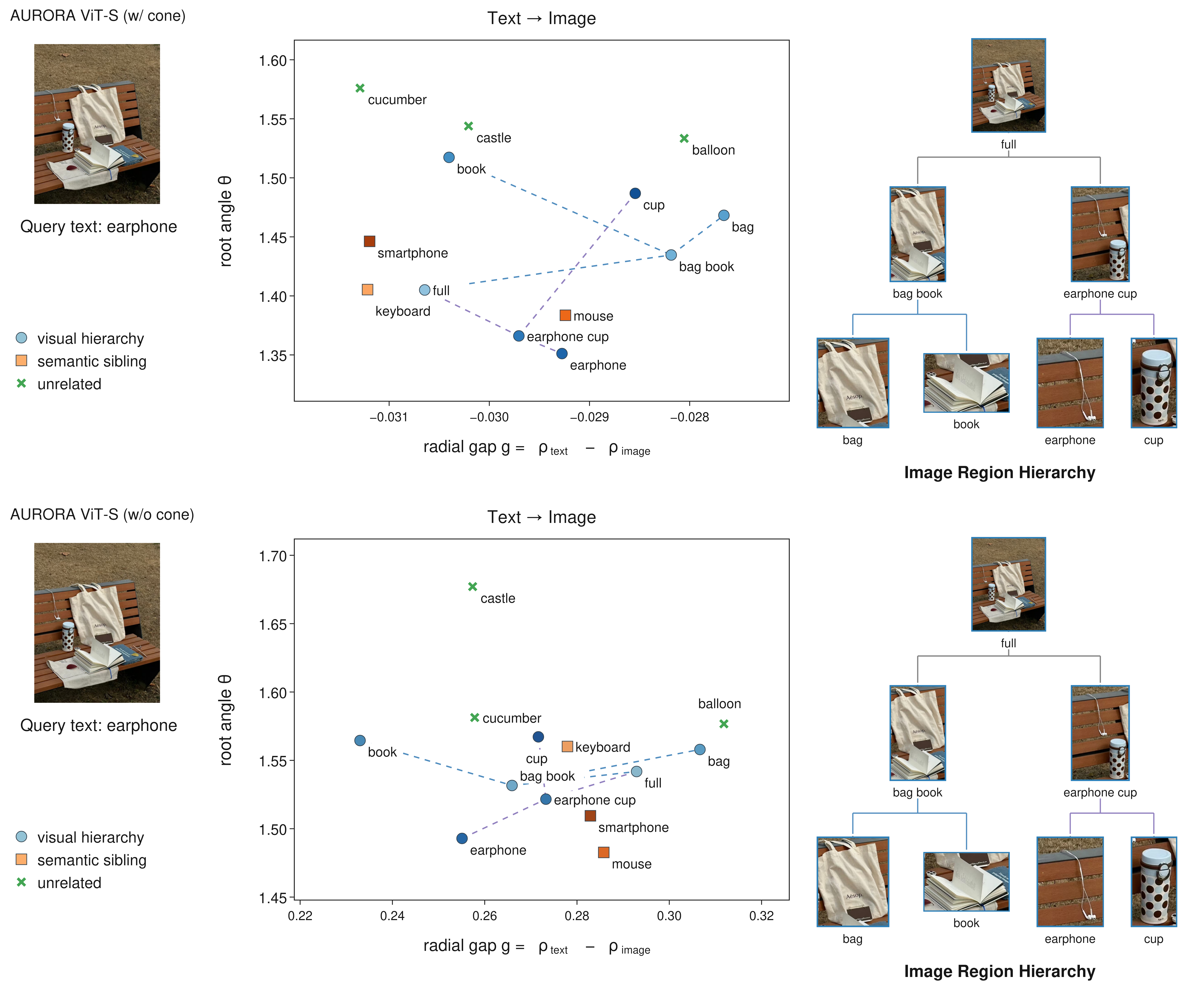}
    \caption{\textbf{Text to image hierarchy for the query ``earphone.''} The top and bottom panels show the models with and without RAOC, respectively. With RAOC, the full image, earphone and cup region, and earphone crop combine progressively smaller root angles with increasing radial gaps. Without RAOC, the same angular trend is present, but the radial sequence is reversed. The comparison separates semantic alignment from the hierarchical ordering encouraged by RAOC.}
    \label{fig:app-t2i-earphone}
\end{figure*}

\begin{figure*}[p]
    \centering
    \includegraphics[width=\textwidth]{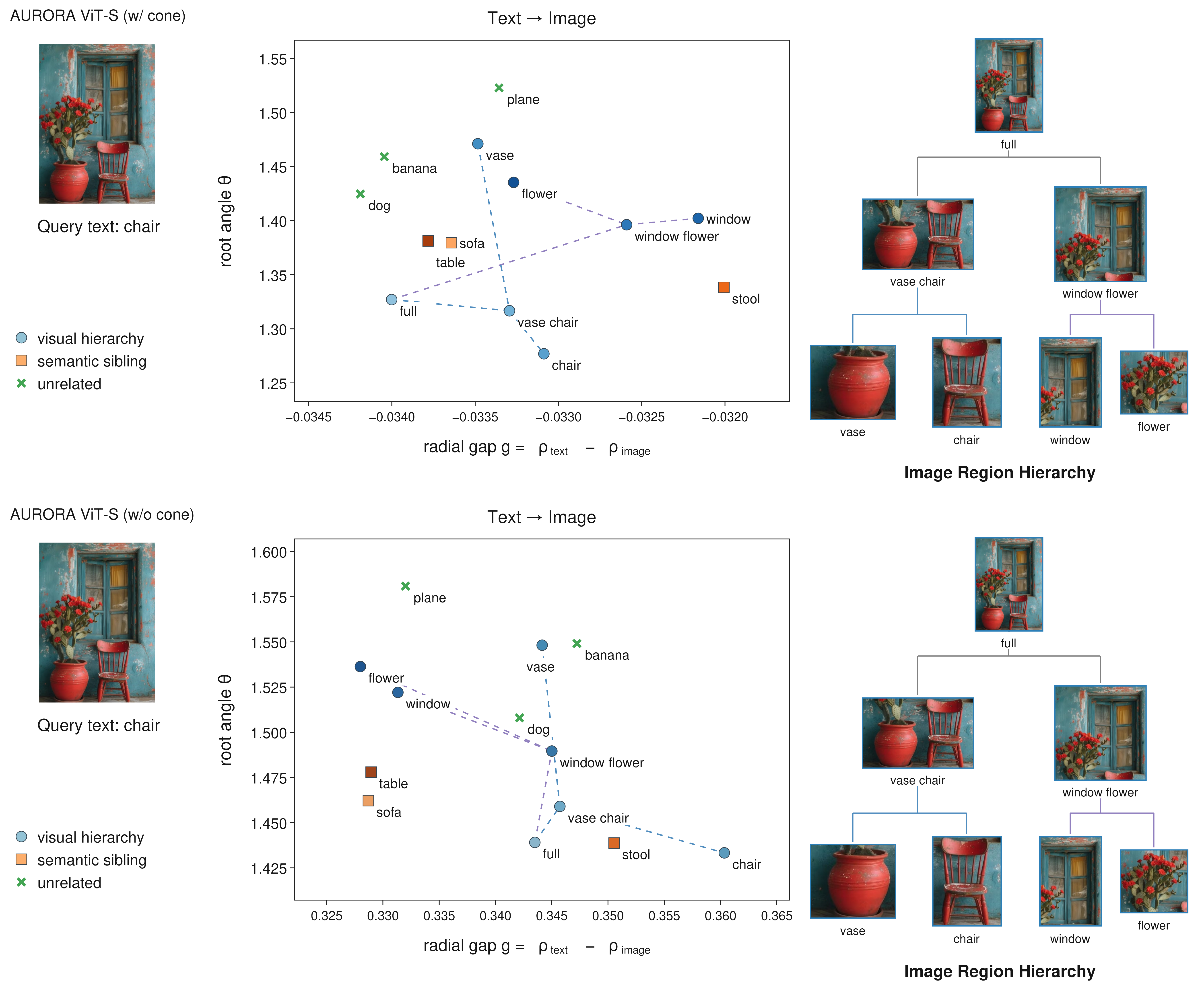}
    \caption{\textbf{Text to image hierarchy for the query ``chair.''} RAOC is used in the top panel and omitted in the bottom panel. Both variants retain radial progression along the full image, vase and chair region, and chair crop. With RAOC, their root angles also decrease consistently. Without RAOC, the intermediate region has a larger root angle than the full image, revealing less consistent angular organization along this branch.}
    \label{fig:app-t2i-chair}
\end{figure*}

\begin{figure*}[p]
    \centering
    \includegraphics[width=\textwidth]{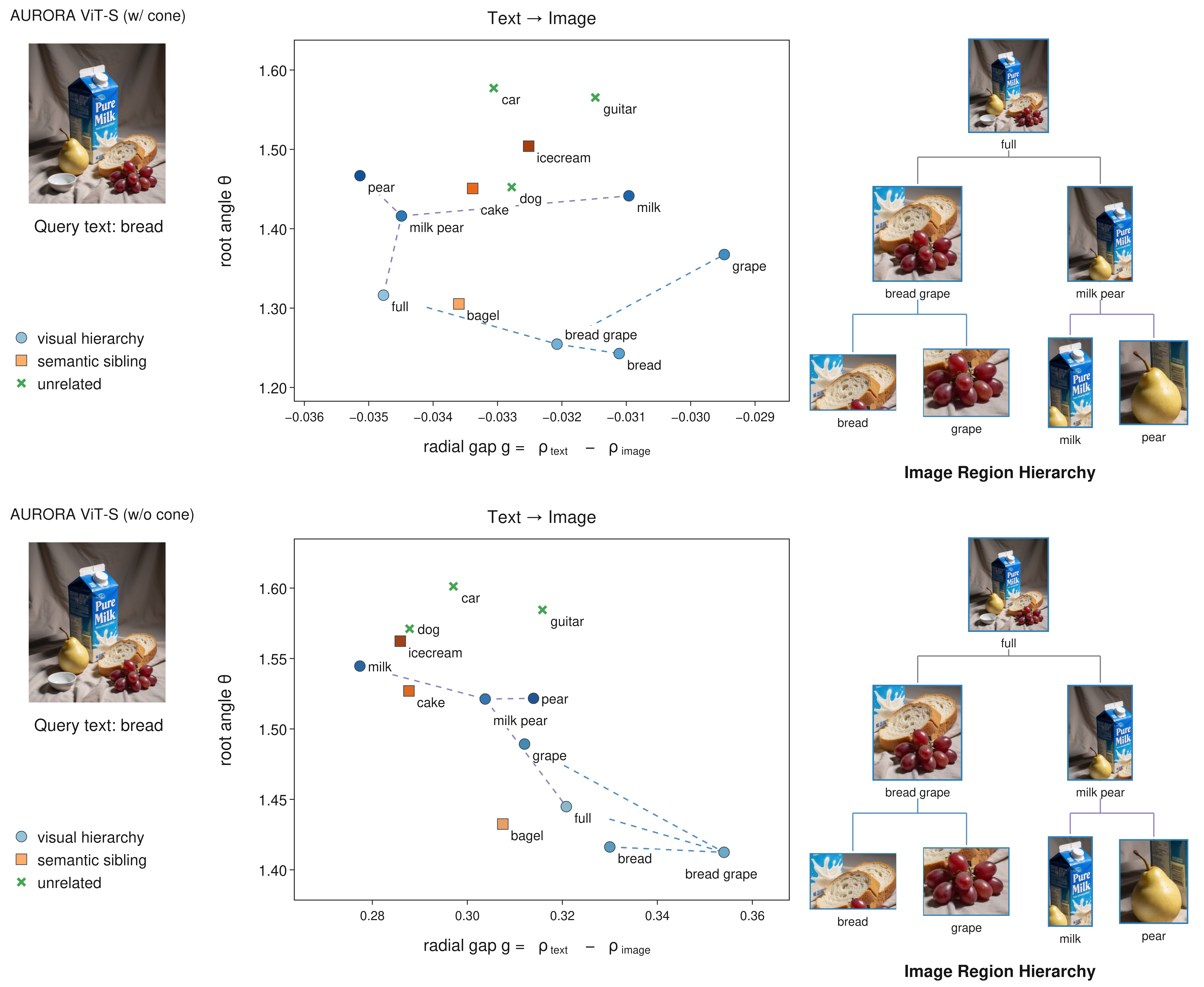}
    \caption{\textbf{Text to image hierarchy for the query ``bread.''} The model with RAOC appears above the model without RAOC. With RAOC, the full image, bread and grape region, and bread crop form a consistent sequence in both radial gap and root angle. Without RAOC, the bread and grape region has a slightly smaller root angle than the bread crop, and the final transition reverses the radial direction. RAOC yields clearer organization along the branch containing the queried object.}
    \label{fig:app-t2i-bread}
\end{figure*}
Figures~\ref{fig:app-t2i-shoes}--\ref{fig:app-t2i-bread} compare five text queries: shoes, eggplant, earphone, chair, and bread. Each figure places the model with RAOC above the model without RAOC. The image region hierarchy on the right provides the same reference structure for both models, and the colored dashed links identify its two branches in the scatter plot. With RAOC, the branch containing the queried object follows an increasing radial gap and a decreasing root angle from the full image through the intermediate region to the object crop. Without RAOC, angular alignment often remains, but the radial progression can reverse or the intermediate region can be less aligned than the full image. These examples provide qualitative evidence that RAOC couples semantic alignment with the radial organization associated with hierarchical partial order.

\section{Hyperparameter Sensitivity}
\label{app:hyperparameter-sensitivity}

We examine the sensitivity of AURORA to four geometric and optimization
parameters: the RAOC angular allowance $\beta$, the entailment weight
$\lambda_e$, the norm regularization weight $\lambda_n$, and the radial
margin scale $\sigma_r$. Each experiment varies one parameter while keeping
the others fixed at the default configuration:
$\beta=8$, $\lambda_e=0.15$, $\lambda_n=0.05$, and
$\sigma_r=0.1$. We use ViT-S/16 and the training setup described in
Appendix~\ref{app:implementation}, evaluating image classification,
COCO retrieval, and hierarchical classification. The default configurations
are highlighted.

\paragraph{RAOC angular allowance.}
The parameter $\beta$ controls the angular allowance associated with a given
radial gap. Smaller values impose stricter angular constraints, whereas
larger values provide more flexibility. Table~\ref{tab:sensitivity-beta}
examines the effect of this geometric parameter.

\begin{table}[H]
\centering
\caption{\textbf{Sensitivity to the RAOC parameter $\beta$.}
All other hyperparameters remain fixed at their default values.}
\label{tab:sensitivity-beta}
\vspace{2pt}
\scriptsize
\setlength{\tabcolsep}{5pt}
\renewcommand{\arraystretch}{1.06}
\begin{tabular}{@{}c *{6}{c}@{}}
\toprule
\cellcolor{vitBand}\textbf{Parameter} &
\multicolumn{2}{c}{\cellcolor{generalHead}\textbf{Classification}} &
\multicolumn{2}{c}{\cellcolor{fineHead}\textbf{COCO R@5}} &
\multicolumn{2}{c}{\cellcolor{specialHead}\textbf{Hierarchy}}\\
\cmidrule(lr){2-3}\cmidrule(lr){4-5}\cmidrule(lr){6-7}
$\beta$ &
\cellcolor{generalCell}ImageNet$\uparrow$ &
\cellcolor{generalCell}Food-101$\uparrow$ &
\cellcolor{fineCell}T$\rightarrow$I$\uparrow$ &
\cellcolor{fineCell}I$\rightarrow$T$\uparrow$ &
\cellcolor{specialCell}TIE$\downarrow$ &
\cellcolor{specialCell}$J\uparrow$\\
\midrule
2  & 41.31 & 49.61 & 55.02 & 68.71 & 3.69 & 0.77\\
4  & 41.18 & 49.82 & 55.31 & 68.97 & 3.67 & 0.77\\
\rowcolor{oursRow}
8  & 41.54 & 50.13 & 55.77 & 69.42 & 3.62 & 0.78\\
16 & 41.12 & 49.74 & 55.18 & 68.84 & 3.71 & 0.77\\
32 & 41.36 & 49.96 & 55.41 & 69.06 & 3.66 & 0.77\\
\bottomrule
\end{tabular}
\end{table}

\paragraph{Entailment weight.}
The weight $\lambda_e$ balances hierarchical entailment supervision and
semantic alignment. Table~\ref{tab:sensitivity-entailment} evaluates the
effect of different supervision strengths.

\begin{table}[H]
\centering
\caption{\textbf{Sensitivity to the entailment weight $\lambda_e$.}
All other hyperparameters remain fixed at their default values.}
\label{tab:sensitivity-entailment}
\vspace{2pt}
\scriptsize
\setlength{\tabcolsep}{5pt}
\renewcommand{\arraystretch}{1.06}
\begin{tabular}{@{}c *{6}{c}@{}}
\toprule
\cellcolor{vitBand}\textbf{Parameter} &
\multicolumn{2}{c}{\cellcolor{generalHead}\textbf{Classification}} &
\multicolumn{2}{c}{\cellcolor{fineHead}\textbf{COCO R@5}} &
\multicolumn{2}{c}{\cellcolor{specialHead}\textbf{Hierarchy}}\\
\cmidrule(lr){2-3}\cmidrule(lr){4-5}\cmidrule(lr){6-7}
$\lambda_e$ &
\cellcolor{generalCell}ImageNet$\uparrow$ &
\cellcolor{generalCell}Food-101$\uparrow$ &
\cellcolor{fineCell}T$\rightarrow$I$\uparrow$ &
\cellcolor{fineCell}I$\rightarrow$T$\uparrow$ &
\cellcolor{specialCell}TIE$\downarrow$ &
\cellcolor{specialCell}$J\uparrow$\\
\midrule
0.0375 & 40.82 & 48.93 & 54.72 & 68.31 & 3.81 & 0.76\\
0.075  & 41.37 & 49.88 & 55.39 & 69.03 & 3.68 & 0.77\\
\rowcolor{oursRow}
0.15   & 41.54 & 50.13 & 55.77 & 69.42 & 3.62 & 0.78\\
0.30   & 40.21 & 48.76 & 54.65 & 68.20 & 3.86 & 0.76\\
0.60   & 39.12 & 46.94 & 53.47 & 66.91 & 4.05 & 0.74\\
\bottomrule
\end{tabular}
\end{table}

\paragraph{Norm regularization weight.}
The weight $\lambda_n$ controls the radial scale regularization. 
Table~\ref{tab:sensitivity-norm} studies how different regularization
strengths affect the learned hierarchy representation.

\begin{table}[H]
\centering
\caption{\textbf{Sensitivity to the norm regularization weight
$\lambda_n$.}
All other hyperparameters remain fixed at their default values.}
\label{tab:sensitivity-norm}
\vspace{2pt}
\scriptsize
\setlength{\tabcolsep}{5pt}
\renewcommand{\arraystretch}{1.06}
\begin{tabular}{@{}c *{6}{c}@{}}
\toprule
\cellcolor{vitBand}\textbf{Parameter} &
\multicolumn{2}{c}{\cellcolor{generalHead}\textbf{Classification}} &
\multicolumn{2}{c}{\cellcolor{fineHead}\textbf{COCO R@5}} &
\multicolumn{2}{c}{\cellcolor{specialHead}\textbf{Hierarchy}}\\
\cmidrule(lr){2-3}\cmidrule(lr){4-5}\cmidrule(lr){6-7}
$\lambda_n$ &
\cellcolor{generalCell}ImageNet$\uparrow$ &
\cellcolor{generalCell}Food-101$\uparrow$ &
\cellcolor{fineCell}T$\rightarrow$I$\uparrow$ &
\cellcolor{fineCell}I$\rightarrow$T$\uparrow$ &
\cellcolor{specialCell}TIE$\downarrow$ &
\cellcolor{specialCell}$J\uparrow$\\
\midrule
0.0125 & 40.71 & 49.12 & 54.89 & 68.54 & 3.76 & 0.76\\
0.025  & 41.08 & 49.77 & 55.34 & 69.01 & 3.69 & 0.77\\
\rowcolor{oursRow}
0.05   & 41.54 & 50.13 & 55.77 & 69.42 & 3.62 & 0.78\\
0.10   & 41.29 & 49.96 & 55.46 & 69.18 & 3.70 & 0.77\\
0.20   & 40.62 & 48.85 & 54.73 & 68.37 & 3.88 & 0.76\\
\bottomrule
\end{tabular}
\end{table}

\paragraph{Radial margin scale.}
The parameter $\sigma_r$ scales the radial margin during optimization.
Table~\ref{tab:sensitivity-radial} evaluates whether performance depends
strongly on this optimization scale.

\begin{table}[H]
\centering
\caption{\textbf{Sensitivity to the radial margin scale $\sigma_r$.}
All other hyperparameters remain fixed at their default values.}
\label{tab:sensitivity-radial}
\vspace{2pt}
\scriptsize
\setlength{\tabcolsep}{5pt}
\renewcommand{\arraystretch}{1.06}
\begin{tabular}{@{}c *{6}{c}@{}}
\toprule
\cellcolor{vitBand}\textbf{Parameter} &
\multicolumn{2}{c}{\cellcolor{generalHead}\textbf{Classification}} &
\multicolumn{2}{c}{\cellcolor{fineHead}\textbf{COCO R@5}} &
\multicolumn{2}{c}{\cellcolor{specialHead}\textbf{Hierarchy}}\\
\cmidrule(lr){2-3}\cmidrule(lr){4-5}\cmidrule(lr){6-7}
$\sigma_r$ &
\cellcolor{generalCell}ImageNet$\uparrow$ &
\cellcolor{generalCell}Food-101$\uparrow$ &
\cellcolor{fineCell}T$\rightarrow$I$\uparrow$ &
\cellcolor{fineCell}I$\rightarrow$T$\uparrow$ &
\cellcolor{specialCell}TIE$\downarrow$ &
\cellcolor{specialCell}$J\uparrow$\\
\midrule
0.025 & 41.08 & 49.74 & 55.21 & 68.92 & 3.70 & 0.77\\
0.05  & 41.32 & 50.01 & 55.54 & 69.18 & 3.65 & 0.77\\
\rowcolor{oursRow}
0.10  & 41.54 & 50.13 & 55.77 & 69.42 & 3.62 & 0.78\\
0.20  & 41.46 & 50.04 & 55.68 & 69.31 & 3.64 & 0.78\\
0.40  & 41.38 & 49.91 & 55.59 & 69.20 & 3.67 & 0.77\\
\bottomrule
\end{tabular}
\end{table}
\section{Reproducibility}
\label{app:repro}

The main paper presents the objectives, methodology, and key experimental
findings. Appendices~\ref{app:implementation} and \ref{app:protocols}
provide implementation details, training configurations, evaluation
protocols, and computational settings. We will release code, checkpoints, 
configurations, data processing details, and evaluation resources,
including scripts for reproducing RAOC-based geometric analyses and
reported benchmark results.
\end{document}